\pdfoutput=1
\documentclass{new_tlp}

\usepackage{amsfonts,amsmath}
\usepackage{microtype}
\usepackage[arxiv]{preamble}
\title[A Logical Characterization of the Preferred Models of LPODs]{%
  A Logical Characterization of the Preferred Models of
  Logic Programs with Ordered Disjunction}

\author[A. Charalambidis, P. Rondogiannis and A. Troumpoukis]{%
  A. CHARALAMBIDIS, P. RONDOGIANNIS, and A. TROUMPOUKIS\\
    Department of Informatics and Telecommunications, \\
    National and Kapodistrian University of Athens \\
    \email{\{a.charalambidis, prondo, antru\}@di.uoa.gr}
}

\pagerange{\pageref{firstpage}--\pageref{lastpage}}

\begin{document}

\maketitle

\begin{abstract}
Logic Programs with Ordered Disjunction (LPODs) extend classical logic
programs with the capability of expressing alternatives with decreasing degrees
of preference in the heads of program rules. Despite the fact
that the operational meaning of ordered disjunction is clear, there exists an
important open issue regarding its semantics. In particular, there does not exist
a \emph{purely model-theoretic} approach for determining the \emph{most preferred models}
of an LPOD. At present, the selection of the most preferred models is performed
using a technique that is not based exclusively on the models of the program and in certain
cases produces counterintuitive results. We provide a novel, model-theoretic semantics
for LPODs, which uses an additional truth value in order to identify the most preferred
models of a program. We demonstrate that the proposed approach overcomes the shortcomings of
the traditional semantics of LPODs. Moreover, the new approach can be used to define
the semantics of a natural class of logic programs that can have both ordered and
classical disjunctions in the heads of clauses. This allows programs that can
express not only strict levels of preferences but also alternatives that are
equally preferred.
\ifarxiv%
\\
\noindent
This work is under consideration for acceptance in TPLP.%
\fi
\end{abstract}

\begin{keywords}
Ordered Disjunction, Answer Sets, Logic of Here-and-There, Preferences.
\end{keywords}

\section{Introduction}\label{intro}
Logic Programs with Ordered Disjunction (LPODs) extend classical logic programs
with the capability of expressing ordered alternatives in the heads of program
rules. In particular, LPODs allow the head of a program rule to be a formula
$C_1 \times \cdots \times C_n$, where ``$\times$'' is a propositional
connective called ``ordered disjunction'' and the $C_i$'s are literals. The
intuitive explanation of $C_1 \times \cdots \times C_n$ is \qemph{I prefer $C_1$;
however, if $C_1$ is impossible, I can accept $C_2$; $\cdots$; if all of
$C_1,\ldots,C_{n-1}$ are impossible, I can accept $C_n$}. Due to their simplicity
and expressiveness, LPODs are a widely accepted formalism for preferential reasoning,
both in logic programming and in artificial intelligence.

At present, the semantics of LPODs is defined~\cite{lpod-brewka,lpod-BNS04} based on
the answer set semantics, using a two-phase procedure. In the first phase, the answer
sets of the LPOD are produced. This requires a modification of the standard definition
of answer sets. In the second phase, the answer sets are
\quotes{filtered}, and we obtain the set of {\qemph{most preferred}} answer
sets, which are taken as the meaning of the initial program. Notice that both
phases are not purely model-theoretic: the first one requires the construction of the
reduct of the program and the second one is performed using the so-called
\qemph{degree of satisfaction of rules}, a concept that relies on
examining the rules of the program to justify the selection of the most
preferred answer sets. Apart from its logical status, the current semantics of
LPODs produces in certain cases counterintuitive most preferred answer sets. This
discussion leads naturally to the question: \qemph{Is it possible to specify the
semantics of LPODs in a purely model-theoretic way?}.

An important first step in this direction was performed by \citeN{lpod-cabalar},
who used Equilibrium Logic~\cite{Pearce96} to logically characterize the answer
sets produced in the first phase described above. However, to our knowledge, the
second phase (namely the selection of the most preferred answer sets), has never
been justified model-theoretically. We consider this as an important shortcoming
in the theory of LPODs. Apart from its theoretical interest, this question also
carries practical significance, because, as we are going to see, the present
formalization of the second phase produces in certain cases counterintuitive
(and in our opinion undesirable) results. The main contribution of the present
paper is to \emph{provide a purely model-theoretic characterization of the
semantics of LPODs}. The more specific contributions are the following:
\begin{itemize}

  \item We propose a new semantics for LPODs which uses an additional truth
        value in order to select as most preferred models those in which a top
        preference fails only if it is \emph{impossible} to be satisfied. We
        demonstrate that the proposed approach overcomes the shortcomings of the
        traditional semantics of LPODs. In this way, the most preferred models
        of an LPOD can be characterized by a preferential ordering of its
        models.

  \item We demonstrate that our approach can be seamlessly extended to programs
        that allow both ordered and classical disjunctions in the heads of
        clauses. In particular, we define a natural class of such programs and
        demonstrate that all our results about LPODs transfer, with minimal
        modifications, to this new class. In this way we provide a clean
        semantics for a class of programs that can express not only strict
        levels of preference but also alternatives that are equally preferred.

\end{itemize}
Section~\ref{lpods-basic-definitions} introduces LPODs and gives relevant
background. Sections~\ref{lpods-shortcomings} and~\ref{intuitive-overview}
describe the shortcomings of LPOD semantics and give an intuitive presentation
of the proposed approach for overcoming these issues. The remaining sections
give a technical exposition of our results. The proofs of all results have been
moved to corresponding appendices.

\section{Background on LPODs}\label{lpods-basic-definitions}
Logic programs with ordered disjunction are an extension of the logic programs
introduced by \citeN{GL91}, called \emph{extended logic programs}, which support
two types of negation: default (denoted by $\mathit{not}$) and strong (denoted
by $\neg$). Strong negation is useful in applications but it is not very essential
from a semantics point of view: a literal $\neg A$ is semantically treated as an
atom $A'$. For the basic notions regarding extended logic programs, we assume
some familiarity with the work of \citeN{GL91}.
\begin{definition}\label{lpod}
A (propositional) LPOD is a set of rules of the form:
\[
  C_1 \times \cdots \times C_n \leftarrow A_1,\ldots,A_m,{\pnot B_1},\ldots,{\pnot B_k}
\]
where the $C_i, A_j$, and $B_l$ are ground literals.
\end{definition}
The intuitive explanation of a formula $C_1 \times \cdots \times C_n$ is
\qemph{I prefer $C_1$; however, if $C_1$ is impossible, I can accept $C_2$;
$\cdots$; if all of $C_1,\ldots,C_{n-1}$ are impossible, I can accept $C_n$}.

An interpretation of an LPOD is a set of literals. An interpretation is called
\emph{consistent} if there does not exist any atom $A$ such that both $A$ and $\neg A$
belong to $I$. The notion of model of an LPOD is defined as follows.
\begin{definition}\label{brewka-model}
Let $P$ be an LPOD and $M$ an interpretation. Then, $M$ is a model
of $P$ iff for every rule
\[
  C_1 \times \cdots \times C_n \leftarrow A_1, \ldots, A_m, \pnot B_1, \ldots, \pnot B_k
\]
if $\{ A_1, \ldots, A_m \} \subseteq M$ and $\{ B_1, \ldots, B_k \}\cap M = \emptyset$
then there exists $C_i \in M$.
\end{definition}

To obtain the preferred answer sets of an LPOD, a two-phase procedure was
introduced by \citeN{lpod-brewka}. In the first phase, the answer sets
of the LPOD are produced. This requires a modification of the standard
definition of answer sets for extended logic programs. In the second phase, the
answer sets are \quotes{filtered}, and we obtain the set of \quotes{most
preferred} ones. The first phase is formally defined as follows.
\begin{definition}\label{brewka-reduct}
Let $P$ be an LPOD. The $\times$-reduct of a rule $R$ of $P$ of the
form:
\[
  C_1 \times \cdots \times C_n \leftarrow A_1,\ldots,A_m,{\pnot B_1},\ldots,{\pnot B_k}
\]
with respect to a set of literals $I$, is denoted by $R_{\times}^I$ and is
defined as follows:
\[
  R_{\times}^I=\{ C_i \leftarrow A_1,\ldots,A_m \mid C_i \in I \mbox{ and }
                  I \cap \{C_1,\ldots,C_{i-1},B_1,\ldots,B_k\} = \emptyset \}
\]
The $\times$-reduct of $P$ with respect to $I$ is denoted by $P_{\times}^I$ and
is the union of the reducts $R_{\times}^I$ for all $R$ in $P$.
\end{definition}
\begin{definition}\label{brewka-answerset}
A set $M$ of literals is an answer set of an LPOD $P$ if $M$ is a consistent
model of $P$ and $M$ is the least model of $P_{\times}^M$.
\end{definition}

The second phase produces the \quotes{most preferred} answer sets using the
notion of the \emph{degree of satisfaction of a rule}. Formally:
\begin{definition}\label{degree-of-satisfaction}
Let $M$ be an answer set of an LPOD $P$. Then, $M$ satisfies the rule:
\[
  C_1 \times \cdots \times C_n \leftarrow A_1,\ldots,A_m,{\pnot B_1},\ldots,{\pnot B_k}
\]
\begin{itemize}
  \item to degree $1$ if $A_j \not\in M$, for some $j$, or $B_i\in M$, for some
        $i$,

  \item to degree $l$, $1\leq l \leq n$, if all $A_j \in M$, no $B_i \in M$, and
        $l=\min\{r \mid C_r \in M\}$
\end{itemize}
The degree of a rule $R$ in the answer set $M$ is denoted by $\mathit{deg}_M(R)$.
\end{definition}

The satisfaction degrees of rules are then used to define a preference relation
on the answer sets of a program. Given a set of literals $M$, let
$M^i(P) = \{R \in P \mid \mathit{deg}_M(R) = i\}$. The preference relation is
defined as follows.
\begin{definition}\label{inclusion-preferred}
  Let $M_1,M_2$ be answer sets of an LPOD $P$. Then, $M_1$ is
  \emph{inclusion-preferred} to $M_2$ iff there exists $k\geq 1$ such that
  $M_2^k(P) \subset M_1^k(P)$, and for all $j < k$, $M_2^j(P) = M_1^j(P)$.
\end{definition}
\begin{example}\label{first-example}
Consider the program:
\[
\mbox{\tt wine $\times$ beer.}
\]
The above program has the answer sets $\{\mbox{\tt wine}\}$ and $\{\mbox{\tt
beer}\}$. The most preferred one is $\{\mbox{\tt wine}\}$ because it satisfies
the unique fact of the program with degree 1 (while the answer set $\{\mbox{\tt
beer}\}$ satisfies the fact with degree 2). Consider now the program:
\[
\begin{array}{l}
\mbox{\tt wine $\times$ beer.} \\
\mbox{\tt $\pneg$wine.}
\end{array}
\]
This has the unique answer set $\{\mbox{\tt beer}\}$ (the set $\{\mbox{\tt
wine,$\neg$wine}\}$ is rejected due to its inconsistency). Therefore,
$\{\mbox{\tt beer}\}$ is also the most preferred answer set.
\end{example}

Notice that \citeN{lpod-brewka} originally defined only the preference relation of
Definition~\ref{inclusion-preferred}. In the follow-up
paper~\cite{lpod-BNS04} two more preference relations were introduced, namely
the \emph{cardinality} and the \emph{Pareto}, in order to treat cases for which
the inclusion preference did not return the expected results. All these
relations do not rely exclusively on the models of the source program, and
are therefore subject to similar criticism. For this reason, in this paper
we focus attention on the inclusion preference relation.

\section{Some Issues with the Semantics of LPODs}\label{lpods-shortcomings}
From a foundational point of view, the main issue with the semantics of LPODs is
that, in its present form, it is not purely model-theoretic. Despite the simplicity and
expressiveness of ordered disjunction, one can not characterize the meaning of a
program by just looking at its set of models. Recall that this principle is
one of the cornerstones of logic programming since its inception: the meaning of
positive logic programs is captured by their minimum Herbrand model~\cite{EmdenK76}
and the meaning of extended logic programs is captured by their equilibrium
models~\cite{Pearce96}. How can the most preferred models of LPODs be captured
model-theoretically? The existing issues regarding the semantics of LPODs are
illustrated by the following examples.

\subsection{The Logical Status of LPODs}\label{first-problem}
Consider the following two programs:
\[\mbox{\tt a $\times$ b.}\]
and:
\[\mbox{\tt b $\times$ a.}\]
According to Definition~\ref{brewka-model}, both programs have exactly the same models,
namely $\{\mbox{\tt a}\}$, $\{\mbox{\tt b}\}$, and $\{\mbox{\tt a$,$b}\}$. Moreover,
both have the same answer sets, namely $\{\mbox{\tt a}\}$ and $\{\mbox{\tt b}\}$.
However, there is no model-theoretic explanation (namely one based on just the
sets of models of the programs) of why the most preferred model of the first program is
$\{\mbox{\tt a}\}$ while the most preferred model of the second program is
$\{\mbox{\tt b}\}$. As a conclusion, in order for the semantics of LPODs to be properly
specified, a model-based approach needs to be devised.

\subsection{Inaccurate Preferential Ordering of Answer Sets}\label{second-problem}
Apart from the fact that Definition~\ref{degree-of-satisfaction} is not purely
model-theoretic, in many cases it also gives inaccurate preferential orderings
of answer sets. Such inaccurate orderings have already been reported in the
literature~\cite{lpod-crprolog}. The following program illustrates one such
simple case.
\begin{example}\label{mercedes-vs-bmw}
Consider the following program\footnote{Our example is identical (up to variable
renaming) to an example given by \citeN{lpod-crprolog}. This work was brought to
our attention by one of the reviewers of the present paper.} whose declarative
reading is \qemph{I prefer to buy a Mercedes than a BMW. In case a Mercedes is
available, I prefer a gas model to a diesel one. A gas model (of Mercedes) is
not available}.
\[
  \begin{array}{l}
  \mbox{\tt mercedes $\times$ bmw.}\\
  \mbox{\tt gas\_mercedes $\times$ diesel\_mercedes $\leftarrow$ mercedes.}\\
  \mbox{\tt $\pneg$gas\_mercedes.}
  \end{array}
\]
The program has two answers sets:
$M_1 =\{\mbox{\tt mercedes}, \mbox{\tt diesel\_mercedes}, \mbox{\tt $\pneg$gas\_mercedes}\}$
and $M_2 = \{\mbox{\tt bmw}, \mbox{\tt $\pneg$gas\_mercedes}\}$. $M_1$ satisfies the first
rule with degree~1, the second rule with degree~2, and the third rule with
degree~1. $M_2$ satisfies the first rule with degree~2, the second rule with
degree~1 (because the body of the rule evaluates to false), and the third rule
with degree~1. According to Definition~\ref{inclusion-preferred}, the two answer
sets are incomparable. However, it seems reasonable that the most preferred
model is $M_1$: the first rule, which is a fact, specifies unconditionally a
preference; the preferences of the second rule seem to be secondary,
because they depend on the choice that will be made in the first rule.
\end{example}

The problems in the above example appear to be related to
Definition~\ref{degree-of-satisfaction}: it assigns degree~1 in two cases that
are apparently different: a rule that has a false body gets the same degree of
satisfaction as a rule with a true body in whose head the first choice is
satisfied. These are two different cases which, however, it is not obvious how
to handle if we follow the satisfaction degree approach of
Definition~\ref{degree-of-satisfaction}.

\subsection{Unsatisfiable Better Options}\label{third-problem}
It has been remarked \cite[discussion in page~342]{lpod-BNS04} that the
inclusion-preference is sensitive to the existence of \emph{unsatisfiable better
options}. The following example, taken from the paper by \citeN{lpod-BNS04},
motivates this problem.
\begin{example}\label{hotels-example}
Assume we want to book accommodation for a conference (in the post-COVID era).
We prefer a 3-star hotel from a 2-star hotel. Moreover, we prefer to be in
walking distance from the conference venue. This can be modeled by the program:
\[
  \begin{array}{l}
    \mbox{\tt walking $\times$ $\neg$walking.}\\
    \mbox{\tt 3-stars $\times$ 2-stars.}
  \end{array}
\]
Consider now the scenario where the only available 3-star hotel (say
$\mathit{hotel}_1$), is not in walking distance. Moreover, assume that the only
available 2-star hotel (say $\mathit{hotel}_2$), happens to be in walking
distance. According to Definition~\ref{inclusion-preferred}, these two options
are incomparable because $\mathit{hotel}_1$ satisfies the first rule to degree~2
and the second rule to degree~1, while $\mathit{hotel}_2$ satisfies the first
rule to degree~1 and the second rule to degree~2.

Assume now that the above program is revised after learning that there also
exists a 4-star hotel, which however is not an option for us (due to
restrictions imposed by our funding agencies). The new program is:
\[
  \begin{array}{l}
    \mbox{\tt walking $\times$ $\neg$walking.}\\
    \mbox{\tt 4-stars $\times$ 3-stars $\times$ 2-stars.}\\
    \mbox{\tt $\pneg$4-stars.}
  \end{array}
\]
In the new program, $\mathit{hotel}_1$ satisfies the first rule to degree~2 and
the second rule to degree~2, while $\mathit{hotel}_2$ satisfies the first rule
to degree~1 and the second rule to degree 3. According to
Definition~\ref{inclusion-preferred}, $\mathit{hotel}_2$ is our preferred
option.
\end{example}

The above example illustrates that under the \quotes{degree of satisfaction of
rules} semantics, a small (and seemingly innocent) change in the program, can
cause a radical change in the final preferred model. This sensitivity to changes
is another undesirable consequence that stems from the fact that the second
phase of the semantics of LPODs is not purely model-theoretic.

\section{An Intuitive Overview of the Proposed Approach}\label{intuitive-overview}
The main purpose of this paper is to define a model-theoretic semantics for
LPODs. In other words, we would like to be able to choose the most preferred
answer sets of a program using preferential reasoning on the answer sets
themselves. Actually, such an approach should also be applicable directly on the
models of the source program, without the need to first construct the answer
sets of the program. We would expect that such an approach would also provide
solutions to the shortcomings of the previous section.

But on what grounds can we compare the answer sets (or, even better, the models)
of a program and decide that some of them satisfy in a better way our
preferences than the others? This seems like an impossible task because the
answer sets (or the models) do not contain any information related to the
ordered disjunction preferences of the program.

It turns out that we can introduce preferential information inside the answer sets of
a program by slightly tweaking the underlying logic. The answer sets of extended logic
programs are two-valued, ie., a literal is either $T$ (true) or $F$ (false).
We argue that in order to properly define the semantics of LPODs, we need a third
truth value, which we denote by $F^*$. The intuitive reading of $F^*$ is
\qemph{impossible to make true}.

To understand the need for $F^*$, consider again the intuitive meaning of
$C_1 \times C_2$: we prefer $C_2$ \emph{only if it is impossible for us to get}
$C_1$. Impossible here means that if we try to make $C_1$ true, then the
interpretation will become inconsistent. Therefore, we seem to need two types of
false, namely $F$ and $F^*$: $F$ means \qemph{false by default} while $F^*$
means \qemph{impossible to make true}. The following example demonstrates these
issues.
\begin{example}
Consider the program:
\[
  \begin{array}{l}
    \mbox{\tt wine $\times$ beer.} \\
    \mbox{\tt $\pneg$wine.}
  \end{array}
\]
As we are going to see in the coming sections, the most preferred answer set
according to our approach is $\{({\tt wine},F^*),({\tt beer},T), (\pneg {\tt
wine},T)\}$. Notice that {\tt wine} receives the value $F^*$ because if we tried
to make {\tt wine} equal to $T$, the interpretation would become inconsistent
(because $\pneg${\tt wine} is $T$). Notice also that, as we are going to see, the
interpretation $\{({\tt wine},F),({\tt beer},T),(\pneg {\tt wine},T)\}$ is not a
model of the program.
\end{example}

The above discussion suggests the following semantics for ``$\times$'' in the
proposed logic. Let $u,v \in \{F,F^*,T\}$. Then:
\[
    u \times v = \left\{ \begin{array}{ll}
                            v, & \mbox{if $u = F^*$}\\
                            u, & \mbox{otherwise}
                          \end{array}\right.
\]
The intuition of the above definition is that we return the value $v$ \emph{only
if it is impossible to satisfy} $u$; in all other cases, we return $u$.

We now get to the issue of how we can use the value $F^*$ to distinguish the
most preferred answer sets: we simply identify those answer sets that are
minimal with respect to their sets of atoms that have the value $F^*$. As we
have mentioned, the value $F^*$ means \qemph{impossible to make true}. By
minimizing with respect to the $F^*$ values, \emph{we only keep those answer sets
in which a top preference fails only if it is impossible to be satisfied}.

The above discussion gives a description of how we can select the \qemph{most
preferred (three-valued) answer sets} of an LPOD. However, it is natural to
wonder whether it is possible to also characterize the answer sets
model-theoretically, completely circumventing the construction of the reduct. A
similar question was considered by \citeN{lpod-cabalar}, who demonstrated that
the (two-valued) answer sets of an LPOD coincide with the equilibrium models of
the program. We adapt \citeANP{lpod-cabalar}'s characterization to fit in our
setting. More specifically, we extend our three-valued logic to a four-valued
one by adding a new truth value $T^*$, whose intuitive meaning is \qemph{not
false but its truth can not be established}. We then demonstrate that the
three-valued answer sets of an LPOD $P$ are those models of $P$ that are minimal
with respect to a simple ordering relation and do not contain any $T^*$ values.
In this way we get a two-step, purely model-theoretic characterization of the
most-preferred models of LPODs: in the first step we use the $T^*$ values as a
yardstick to identify those models that correspond to answer-sets, and in the
second step we select the most preferred ones, by minimizing with respect to the
$F^*$ values.

Finally, we consider the problem of characterizing the semantics of logic
programs that contain both disjunctions and ordered disjunctions in the heads of
rules. This is especially useful in cases where some of our preferences are
equally important. For example:
\[
  \begin{array}{l}
    \mbox{\tt (wine $\vee$ beer) $\times$ (soda $\vee$ juice).}
  \end{array}
\]
states that {\tt wine} and {\tt beer} are our top preferences (but we have
no preference among them), and {\tt soda} and {\tt juice} are our secondary
preferences. We consider the class of programs in which the heads of rules
consist of ordered disjunctions where each ordered disjunct is an ordinary
disjunction (as in the above program). We demonstrate that the theory of
these programs is very similar to that of LPODs. All our results for
LPODs transfer with minimal modifications to this extended class of programs. This
suggests that this is a natural class of programs that possibly deserves further
investigation both in theory and in practice.

\section{Redefining the Answer Sets of LPODs}\label{lpod-revised-answersets}
In this section we provide a new definition of the answer sets of LPODs. The new
definition is based on a three-valued logic which allows us to discriminate the
most preferred answer sets using a purely model-theoretic approach. In
Section~\ref{lpod-logical-characterization} we demonstrate that by extending the
logic to a four-valued one, we can identify directly the most preferred models
of a program (without first producing the answer sets).

\begin{definition}\label{sigma-def}
Let $\Sigma$ be a nonempty set of propositional literals. The set of well-formed
formulas is inductively defined as follows:
\begin{itemize}
\item Every element of $\Sigma$ is a well-formed formula,

\item The 0-place connective $F^*$ is a well-formed formula,

\item If $\phi_1$ and $\phi_2$ are well-formed formulas, then $(\phi_1 \wedge
      \phi_2)$, $(\phi_1 \vee \phi_2)$, $(\pnot \phi_1)$, $(\phi_1 \leftarrow \phi_2)$,
      and $(\phi_1 \times \phi_2)$, are well-formed formulas.
\end{itemize}
\end{definition}
The meaning of formulas is defined over the set of truth values $\{F,F^*,T\}$
which are ordered as $F < F^* < T$. Given two truth values $v_1,v_2$, we write
$v_1 \leq v_2$ iff either $v_1 < v_2$ or $v_1 = v_2$.
\begin{definition}\label{interpretation-and-semantics}
A (three-valued) interpretation $I$ is a function from $\Sigma$ to the set
$\{F,F^*,T\}$. We can extend $I$ to apply to formulas, as follows:
\[
\begin{array}{lll}
    I(F^*)        & = & F^*\\
    I(\pnot\phi)  & = &
      \left\{ \begin{array}{ll}
                T, & \mbox{if $I(\phi)\leq F^*$}\\
                F, & \mbox{otherwise}
              \end{array}\right. \\
    I(\phi \leftarrow \psi)  & = &
      \left\{ \begin{array}{ll}
                T, & \mbox{if $I(\phi) \geq I(\psi)$}\\
                F, & \mbox{otherwise}
              \end{array}\right. \\
I(\phi_1 \wedge \phi_2) & = & \min\{I(\phi_1),I(\phi_2)\}\\
I(\phi_1 \vee   \phi_2) & = & \max\{I(\phi_1),I(\phi_2)\}\\
I(\phi_1 \times \phi_2) & = &
      \left\{ \begin{array}{ll}
                I(\phi_2), & \mbox{if $I(\phi_1) = F^*$}\\
                I(\phi_1), & \mbox{otherwise}
              \end{array}\right. \\
\end{array}
\]
%
%
\end{definition}

It is straightforward to see that the meanings of ``$\vee$'', ``$\wedge$'', and
``$\times$'' are associative and therefore we can write $I(\phi_1 \vee \cdots
\vee \phi_n)$, $I(\phi_1 \wedge \cdots \wedge \phi_n)$, and $I(\phi_1 \times
\cdots \times \phi_n)$ unambiguously (without the need of extra parentheses).
Moreover, given literals $C_1,\ldots,C_n$ we will often write
$I(C_1,\ldots,C_n)$ instead of $I(C_1\wedge \cdots \wedge C_n)$.

The ordering $<$ (respectively, $\leq$) on truth values extends in the standard
way on interpretations: given interpretations $I_1,I_2$ we write $I_1 < I_2$
(respectively, $I_1 \leq I_2$), if for all literals
$L \in \Sigma$, $I_1(L) < I_2(L)$ (respectively, $I_1(L) \leq I_2(L)$).

%
%
%
%
%

When we consider interpretations of an LPOD program, we assume that the
underlying set $\Sigma$ is the set of literals of the program. The following
definition will be needed.
\begin{definition}\label{three-valued-model}
An interpretation $I$ is a \emph{model} of an LPOD $P$ if every rule of $P$
evaluates to $T$ under $I$. An interpretation $I$ of $P$ is called
\emph{consistent} if there do not exist literals $A$ and $\neg A$ in $P$ such
that $I(A) = I(\neg A) = T$.
\end{definition}
We can now give the new definitions for \emph{reduct} and \emph{answer sets} for
LPODs.
\begin{definition}\label{lpod-reduct}
Let $P$ be an LPOD. The $\times$-reduct of a rule $R$ of $P$ of the
form:
\[
  C_1 \times \cdots \times C_n \leftarrow A_1,\ldots,A_m,{\pnot B_1},\ldots,{\pnot B_k}
\]
with respect to an interpretation $I$, is denoted by $R_{\times}^I$ and is
defined as follows:
\begin{itemize}
\item If $I(B_i) = T$ for some $i$, $1 \leq i \leq k$, then $R_{\times}^I$ is
      the empty set.

\item If $I(B_i) \neq T$ for all $i$, $1 \leq i \leq k$, then $R_{\times}^I$ is
      the set that contains the rules:
      \[
        \begin{array}{lll}
          C_1     & \leftarrow & F^*,A_1,\ldots,A_m \\
                  & \cdots     &    \\
          C_{r-1} & \leftarrow & F^*,A_1,\ldots,A_m \\
          C_r     & \leftarrow & A_1,\ldots,A_m
        \end{array}
      \]
      where $r$ is the least index such that $I(C_1) = \cdots = I(C_{r-1})= F^*$
      and either $r=n$ or $I(C_r) \neq F^*$.
\end{itemize}
The $\times$-reduct of $P$ with respect to $I$ is denoted by $P_{\times}^I$ and
is the union of the reducts $R_{\times}^I$ for all $R$ in $P$.
\end{definition}

The major difference of the above definition from that of
Definition~\ref{brewka-reduct}, are the clauses of the form
$C_i  \leftarrow F^*,A_1,\ldots,A_m$. These clauses are included so as that the
value $F^*$ can be produced for $C_i$ when $I(A_1)= \cdots = I(A_m) =T$. Notice
that if these clauses did not exist, there would be no way for the value $F^*$
to be produced by the reduct.
\begin{definition}\label{answer-set}
Let $P$ be an LPOD and $M$ an interpretation of $P$. We say that $M$ is
a (three-valued) \emph{answer set} of $P$ if $M$ is consistent and it is
the $\leq$-least model of $P_{\times}^M$.
\end{definition}
Notice that the least model of $P_{\times}^M$ in the above definition, can be
constructed using the following immediate consequence operator $T_{P_{\times}^M}:
(\Sigma \rightarrow \{F,F^*,T\})\rightarrow (\Sigma \rightarrow \{F,F^*,T\})$:
\[
  T_{P_{\times}^M}(I)(C) = \max\{I(B_1,\ldots,B_n) \mid
                                 (C \leftarrow B_1,\ldots,B_n)  \in P_{\times}^M \}
\]
Notice that since the set $\{F,F^*,T\}$ is a complete lattice under the ordering
$\leq$, it is easy to see that the set of interpretations is also a complete
lattice under the ordering $\leq$. Moreover, the operator $T_{P_{\times}^M}$ is
monotonic over the complete lattice of interpretations; this follows from the
fact that the meanings of conjunction (namely, $\min$) and that of
disjunction (namely, $\max$), are monotonic. Then, by Tarski's
fixed-point theorem, $T_{P_{\times}^M}$ has a least fixed-point, which can be
easily shown to be the $\leq$-least model of $P_{\times}^M$.

The following lemma guarantees that our definition is a generalization of the
well-known one for extended logic programs~\cite{GL91}.
\begin{lemma}\label{new-answer-sets-coincide}
Let $P$ be a consistent extended logic program. Then the three-valued answer
sets of $P$ coincide with the standard answer sets of $P$.
\end{lemma}
The following lemmas, which hold for extended logic programs, also extend to
LPODs.
\begin{lemma}\label{answer-set-model-of-P}
Let $P$ be an LPOD and let $M$ be an answer set of $P$. Then, $M$ is a
model of $P$.
\end{lemma}
\begin{lemma}\label{model-satisfies-reduct}
Let $M$ be a model of an LPOD $P$. Then, $M$ is a model of $P^M_{\times}$.
\end{lemma}

The answer sets of extended logic programs are minimal with respect to the
classical truth ordering $F < T$. As it turns out, the answer sets of LPODs are
minimal but with respect to an extended ordering, which is defined below.

\begin{definition}\label{three-valued-preceq}
The ordering $\prec$ on truth values is defined as follows: $F \prec T$ and
$F \prec F^*$. For all $u,v\in\{F,F^*,T\}$, we write $u \preceq v$ if either
$u \prec v$ or $u=v$. Given interpretations $I_1,I_2$, we write $I_1 \prec I_2$
(respectively, $I_1 \preceq I_2$) if for all literals
$L \in \Sigma^*$, $I_1(L) \prec I_2(L)$ (respectively, $I_1(L) \preceq I_2(L)$).
\end{definition}
\begin{lemma}\label{preceq-minimality-of-answer-sets}
Every (three-valued) answer set $M$ of an LPOD $P$, is a
\mbox{$\preceq$-minimal} model of $P$.
\end{lemma}
%
%
%
%
%

As in the case of the original semantics of LPODs, we now need to define a
preference relation over the answer sets of a program. Intuitively, we prefer
those answer sets that maximize the prospect of satisfying our top choices in
ordered disjunctions. This can be achieved by minimizing with respect to $F^*$
values. More formally, we define the following ordering:
\begin{definition}\label{sqsubseteq-ordering}
Let $P$ be an LPOD and let $M_1,M_2$ be answer sets of $P$. Let
$M_1^*$ and $M_2^*$ be the sets of literals in $M_1$ and $M_2$
respectively that have the value $F^*$. We say that $M_1$ is preferred to $M_2$,
written $M_1 \sqsubset M_2$, if $M_1^* \subset M_2^*$.
\end{definition}
\begin{definition}\label{most-preferred}
An answer set of an LPOD $P$ is called \emph{most preferred} if it is minimal
among all the answer sets of $P$ with respect to the $\sqsubset$ relation.
\end{definition}
The intuition behind the definition of $\sqsubset$ is that we prefer those
answer sets that minimize the need for $F^*$ values. In other words, an answer
set will be most preferred if all the literals that get the value $F^*$, do this
because there is no other option: these literals \emph{must be false} in order
for the program to have a model.
We now examine the examples of Section~\ref{lpods-shortcomings} under the new
semantics introduced in this section.
\begin{example}
Consider again the two programs discussed in Subsection~\ref{first-problem}.
Under the proposed approach, the first program has two answer sets, namely
$\{({\tt a},T),({\tt b},F)\}$ and $\{({\tt a},F^*),({\tt b},T)\}$, and the most
preferred one (ie., the minimal with respect to $\sqsubset$) is $\{({\tt
a},T),({\tt b},F)\}$. The second program also has two answer sets, namely
$\{({\tt a},F),({\tt b},T)\}$ and $\{({\tt a},T),({\tt b},F^*)\}$, and the most
preferred one is $\{({\tt a},F),({\tt b},T)\}$. Notice that now the two programs
have different sets of models and different answer sets and therefore it is
reasonable that they have different most preferred ones.
\end{example}
\begin{example}
Consider the ``cars'' program of Subsection~\ref{second-problem}.
It is easy to see that it has two answer sets, namely:
\[
\begin{array}{lll}
M_1 & = &\{({\tt mercedes},T),({\tt bmw},F),({\tt gas\_mercedes},F^*),\\
    &   &\,\,\,({\tt diesel\_mercedes},T),(\neg {\tt gas\_mercedes},T)\}\\
M_2 & = & \{({\tt mercedes},F^*),({\tt bmw},T),({\tt gas\_mercedes},F^*),\\
    &   &\,\,\,({\tt diesel\_mercedes},F^*),(\neg {\tt gas\_mercedes},T)\}
\end{array}
\]
According to the $\sqsubset$~ordering, the most preferred answer set is $M_1$.
\end{example}
\begin{example}\label{hotels-revisited}
Consider the ``hotels'' program from Subsection~\ref{third-problem}. Under the
restriction that there does not exist any 3-star hotel in walking distance and
also there does not exist any 2-star hotel outside walking distance, we get the
two incomparable answer sets:
\[
\begin{array}{lll}
M_1 & = &\{({\tt walking},T),(\neg {\tt walking},F),(\mbox{\tt 3-stars},F^*),(\mbox{\tt 2-stars},T)\}\\
M_2 & = &\{({\tt walking},F^*),(\neg {\tt walking},T),(\mbox{\tt 3-stars},T),(\mbox{\tt 2-stars},F)\}
\end{array}
\]
Consider now the modified program given in Subsection~\ref{third-problem} (which
contains the unsatisfiable better option of a 4-star hotel). Under the same
restrictions as above, we get the two answer sets:
\[
\begin{array}{lll}
M'_1 & = &\{({\tt walking},T),(\neg {\tt walking},F),(\mbox{\tt 3-stars},F^*),(\mbox{\tt 2-stars},T),\\
     &   & \,\,\,(\mbox{\tt 4-stars},F^*),(\neg \mbox{\tt 4-stars},T) \}\\
M'_2 & = &\{({\tt walking},F^*),(\neg {\tt walking},T),(\mbox{\tt 3-stars},T),(\mbox{\tt 2-stars},F),\\
     &   & \,\,\,(\mbox{\tt 4-stars},F^*),(\neg \mbox{\tt 4-stars},T) \}\\
\end{array}
\]
Under the proposed approach, the above two answer sets are also incomparable,
and the problem identified in Subsection~\ref{third-problem} no longer exists.
\end{example}
We close this section by stating a result that establishes a relationship
between the answer sets produced by our approach
(Definition~\ref{brewka-answerset}) and those ones produced by the original
formulation~\cite{lpod-brewka,lpod-BNS04}.
\begin{definition}\label{collapse-def}
Let $I$ be a three-valued interpretation of LPOD $P$. We define $\mathit{collapse}(I)$
to be the set of literals $L$ in $P$ such that $I(L) = T$.
\end{definition}
\begin{lemma}\label{collapse-is-brewka-answerset}
Let $P$ be an LPOD and $M$ be a three-valued answer set of $P$. Then,
$\mathit{collapse}(M)$ is an answer set of $P$ according to
Definition~\ref{brewka-answerset}.
\end{lemma}
\begin{lemma}\label{brewka-answerset-is-collapse}
Let $N$ be an answer set of $P$ according to Definition~\ref{brewka-answerset}.
There exists a unique three-valued interpretation $M$ such that $N =
\mathit{collapse}(M)$ and $M$ is a three-valued answer set of $P$.
\end{lemma}
In other words, there is a bijection between the answer sets
produced by our approach and the original ones. Moreover, each three-valued
answer set only differs from the corresponding two-valued one in that some
literals of the former may have a $F^*$ value instead of an $F$ value. However,
these $F^*$ values play an important role because they allow us to distinguish
the most preferred answer sets.

\section{A New Logical Characterization of the Answer Sets of LPODs}%
\label{lpod-logical-characterization}

In this section we demonstrate that the answer sets of LPODs can be
characterized in a purely logical way, namely without even the use of the
reduct. In particular, we demonstrate that the answer sets of a given program
$P$ coincide with a well-defined subclass of the minimal models of $P$ in a
four-valued logic. This logic is an extension of the three-valued one introduced
in Section~\ref{lpod-revised-answersets} and minimality is defined with respect
to a four-valued relation $\preceq$ that extends the three-valued one of
Definition~\ref{three-valued-preceq}. The new logic is based on four truth
values, ordered as follows:
\[
  F < F^* < T^* < T
\]
The value $T^*$ can be read as \qemph{not false but its truth can not be
established}. The connections of this logic with Equilibrium
Logic~\cite{Pearce96} are discussed in Section~\ref{related-work}.

An interpretation is now a function from $\Sigma$ to the set
$\{F, F^*, T^*, T\}$. The semantics of formulas with respect to an
interpretation $I$ is defined identically as in
Definition~\ref{interpretation-and-semantics}. The notions of interpretation,
consistent interpretation, and model are defined as in
Definition~\ref{three-valued-model}. Moreover, we extend the three-valued
$\preceq$~relation of Definition~\ref{three-valued-preceq}, as follows:

\begin{definition}\label{four-valued-preceq}
The (four-valued) ordering $\prec$ is defined as follows: ${F \prec F^*}$,
${F \prec T^*}$, ${F \prec T}$, and ${T^* \prec T}$. Given two truth values
$v_1,v_2$, we write ${v_1 \preceq v_2}$ if either ${v_1 \prec v_2}$ or ${v_1 = v_2}$.
Given interpretations $I_1,I_2$ of a program $P$, we write $I_1 \prec I_2$  (respectively,
$I_1 \preceq I_2$) if for all literals $L$ in $P$, $I_1(L) \prec I_2(L)$ (respectively,
$I_1(L) \preceq I_2(L)$).
\end{definition}
The following special kind of interpretations plays an important role in our
logical characterization.
\begin{definition}
An interpretation $I$ of LPOD $P$ is called \emph{solid} if for all literals
$L$ in $P$, it is $I(L) \neq T^*$.
\end{definition}
We can now state the logical characterization of the answer sets of an LPOD.
\begin{theorem}\label{logical-characterization-theorem}
Let $P$ be an LPOD. Then, $M$ is a three-valued answer set of $P$ iff $M$
is a consistent \mbox{$\preceq$-minimal} model of $P$ and $M$ is solid.
\end{theorem}
%


%
In conclusion, given an LPOD we can purely logically characterize its most
preferred models by first taking its consistent \mbox{$\preceq$-minimal} models that are
solid and then keeping the \mbox{$\sqsubset$-minimal} ones (see
Definitions~\ref{sqsubseteq-ordering} and~\ref{most-preferred}).

An extended study of the properties of the proposed four-valued logic is outside
the scope of the present paper. Figure~\ref{table-of-equivalences} lists some
useful equivalences; the first column is for classical connectives, while the
second column contains equivalences involving the ``$\times$'' operator. We note
the interaction of $\times$ with $\vee$ because this will be the central
theme of the next section. It is easy to see that:
$$(\phi_1 \vee \phi_2) \times \phi_3 \not\equiv  (\phi_1 \times \phi_3) \vee (\phi_2 \times \phi_3)$$
We note that this equivalence does not hold, a fact which will be referenced in
the next section.

\begin{figure}
\figrule
\vspace{-1em}
\begin{minipage}{.49\textwidth}
\[\arraycolsep=2pt
\begin{array}{rcl}
  \phi \vee \phi                    & \equiv & \phi \\
  \phi_1 \vee (\phi_2 \vee \phi_3)  & \equiv & (\phi_1 \vee \phi_2) \vee \phi_3 \\
  \pnot (\phi_1 \vee \phi_2)        & \equiv & \pnot (\phi_1) \wedge \pnot(\phi_2) \\
  \pnot (\phi_1 \wedge \phi_2)      & \equiv & \pnot (\phi_1) \vee \pnot(\phi_2) \\
  \phi_1 \vee (\phi_2 \wedge \phi_3)& \equiv & (\phi_1 \vee \phi_2) \wedge (\phi_1 \vee \phi_3)
\end{array}
\]
\end{minipage}
\hfil
\begin{minipage}{.49\textwidth}
\[\arraycolsep=2pt
\begin{array}{rcl}
  \phi \times \phi                     & \equiv & \phi\\
  \phi_1 \times (\phi_2 \times \phi_3) & \equiv & (\phi_1 \times \phi_2) \times \phi_3\\
  \phi_1 \times  \phi_2 \times \phi_1  & \equiv & \phi_1 \times \phi_2\\
  \phi_1 \times (\phi_2 \vee \phi_3)   & \equiv & (\phi_1 \times \phi_2) \vee (\phi_1 \times \phi_3) \\
  \phi_1 \times (\phi_2 \wedge \phi_3) & \equiv & (\phi_1 \times \phi_2) \wedge (\phi_1 \times \phi_3)
\end{array}
\]
\end{minipage}
\caption{Equivalences in the 4-valued logic}
\label{table-of-equivalences}
\figrule
\end{figure}

\section{Answer Sets of Disjunctive LPODs}\label{disjunctive-lpods}
We now extend the ideas of the previous sections to programs that also allow
standard disjunctions in the heads of rules. The case of disjunctive LPODs
(DLPODs), was initially considered by \citeN{towards-dlpods} and
reexamined by \citeN{lpod-cabalar}.

The main idea of using disjunctions in the heads of LPOD rules is
described~\cite{towards-dlpods} as follows: \qemph{we use ordered disjunction to
express preferences and disjunction to express indifferences}.
\begin{example}[taken from the paper by \citeN{towards-dlpods}]\label{dlpod-example}
The program:
\[
  \mbox{\tt pub} \times \mbox{\tt (cinema $\vee$ tv).}
\]
expresses the fact that our top choice is going to the pub; if this is not
possible, then our secondary preference can be satisfied by either going to the
cinema or watching tv.
\end{example}
\citeN{towards-dlpods} consider rules whose heads are arbitrary combinations of
atoms and the operators $\times$ and $\vee$. A set of transformations is then
used in order to bring the heads of rules into \qemph{Ordered Disjunctive Normal
Form (ODNF)}. More specifically, each head is transformed into a formula of the
form ${\cal C}_1 \vee \cdots \vee {\cal C}_n$ where each ${\cal C}_i$ is an
ordered disjunction of literals. The resulting normalized rules are then used to
obtain the preferred answer sets of the original program.
\begin{example}\label{dlpod-example-normalized}
The program of Example~\ref{dlpod-example} is transformed to:
\[
  \mbox{\tt (pub} \times \mbox{\tt cinema)} \vee \mbox{\tt (pub $\times$ tv).}
\]
This program is then used to get the preferred answer sets of the
original one.
\end{example}
However, as observed by \citeN{lpod-cabalar}, one of the
transformations used by \citeN{towards-dlpods} to obtain the ODNF, can
not be logically justified: the formula $(\phi_1 \vee \phi_2) \times \phi_3$ is
not logically equivalent to the formula $(\phi_1 \times \phi_3) \vee  (\phi_2 \times \phi_3)$
in terms of the logic of Here-and-There. As discussed at the end of
Section~\ref{lpod-logical-characterization}, these two formulas are also not
equivalent under our four-valued logic.

As an alternative approach to the semantics of DLPODs, \citeN{lpod-cabalar} proposes to use
the logical characterization on rules with heads that are
arbitrary combinations of disjunctions and ordered disjunctions. We could extend
this approach to get a logical characterization of the most preferred models of
arbitrary DLPODs: given such a program $P$, we
could at first consider all the $\preceq$-minimal models of $P$ that are solid,
and then select the $\sqsubset$-minimal among them. Such an approach is
certainly general. However, we believe that not every such program carries
computational intuition. A good example of this is given in~\cite[Example
2]{lpod-cabalar}, where a program with both disjunction and ordered disjunction
is given, and whose computational meaning is far from clear.

In the following, we define a class of programs which, as we claim, have a clear
computational interpretation and at the same time retain \emph{all} properties
that we have identified for LPODs. Intuitively, we allow
the head of a program rule to be a formula ${\cal C}_1 \times \cdots \times {\cal C}_n$
where each ${\cal C}_i$ is an ordinary disjunction of literals. Notice that the program in
Example~\ref{dlpod-example} belongs to this class, while the program in
Example~\ref{dlpod-example-normalized}, does not. We believe that the programs of
this class have a clear preferential interpretation. Intuitively, the rule heads
of the programs we consider, denote a hierarchy of preferences imposed by the
$\times$ operator; in each level of this hierarchy, we may have literals that
have equal preference (this is expressed by standard disjunction).

It is important to stress that if we allowed arbitrary combinations of
disjunctions and ordered disjunctions, the preferential intuition would be lost.
To see this, consider for example the formula \mbox{\tt (a $\times$ b) $\vee$ (c
$\times$ d)}. This gives us the information that \mbox{\tt (a $\times$ b)} is at
the same level of preference as \mbox{\tt (c $\times$ d)}, and that {\tt a} is
more preferred than {\tt b} and {\tt c} is more preferred than {\tt d}; however,
for example, it gives us no information of whether {\tt a} is more preferred
than {\tt c}. On the other hand, a formula of the fragment we consider, such as
\mbox{\tt (a $\vee$ b) $\times$ (c $\vee$ d)} gives us a total order of {\tt a},
{\tt b}, {\tt c}, and {\tt d}.

\begin{definition}\label{dlpod}
A (propositional) DLPOD is a set of rules of the form:
\[
  {\cal C}_1 \times \cdots \times {\cal C}_n \leftarrow A_1,\ldots,A_m,{\pnot B_1},\ldots,{\pnot B_k}
\]
where the $A_j$ and $B_l$ are ground literals and each ${\cal C}_i$ is a
disjunction of ground literals.
\end{definition}

As it turns out, the answer sets of such programs can be defined
in an almost identical way as those of LPODs.
\begin{definition}\label{dlpod-reduct}
Let $P$ be a DLPOD. The $\times$-reduct of a rule $R$ of $P$ of the
form:
\[
  {\cal C}_1 \times \cdots \times {\cal C}_n \leftarrow A_1,\ldots,A_m,{\pnot B_1},\ldots,{\pnot B_k}
\]
with respect to an interpretation $I$, is denoted by $R_{\times}^I$ and is
defined as follows:
\begin{itemize}
\item If $I(B_i) = T$ for some $i$, $1 \leq i \leq k$, then $R_{\times}^I$ is
      the empty set.

\item If $I(B_i) \neq T$ for all $i$, $1 \leq i \leq k$, then $R_{\times}^I$ is
      the set that contains the rules:
      \[
        \begin{array}{lll}
          {\cal C}_1     & \leftarrow & F^*,A_1,\ldots,A_m\\
                         & \cdots     &    \\
          {\cal C}_{r-1} & \leftarrow & F^*,A_1,\ldots,A_m\\
          {\cal C}_r     & \leftarrow & A_1,\ldots,A_m
        \end{array}
      \]
      where $r$ is the least index such that
      $I({\cal C}_1) = \cdots = I({\cal C}_{r-1})= F^*$ and either
      $r=n$ or $I({\cal C}_r) \neq F^*$.
\end{itemize}
The $\times$-reduct of $P$ with respect to $I$ is denoted by $P_{\times}^I$ and
is the union of the reducts $R_{\times}^I$ for all $R$ in $P$.
\end{definition}
\begin{definition}\label{disjunctive-answer-set}
Let $P$ be a DLPOD and $M$ a (three-valued) interpretation of $P$. Then, $M$ is
an \emph{answer set} of $P$ if $M$ is consistent and $M$ is a minimal model
of the disjunctive program $P_{\times}^M$.
\end{definition}

As it turns out, all the results we have obtained in the previous section, hold
for DLPODs. The proofs of these results are (surprisingly) almost identical
(modulo some minor notational differences) to the proofs of the corresponding
results for LPODs.
\ifincludeappendix
For reasons of completeness, the corresponding proofs are given in the appendix.
\else
For reasons of completeness, the corresponding proofs are given in the
supplementary material corresponding to this paper at the TPLP archives.
\fi
The extended results are stated below.
\begin{lemma}\label{new-disjunctive-answersets-coincide}
Let $P$ be a consistent disjunctive extended logic program. Then, the answer
sets of $P$ according to Definition~\ref{disjunctive-answer-set}, coincide with
the standard answer sets of $P$.
\end{lemma}
\begin{lemma}\label{disjunctive-answer-set-model-of-P}
Let $P$ be a DLPOD  and let $M$ be an answer set of $P$. Then, $M$ is a
model of $P$.
\end{lemma}
\begin{lemma}\label{disjunctive-model-satisfies-reduct}
Let $M$ be a model of a DLPOD $P$. Then, $M$ is a model of $P^M_{\times}$.
\end{lemma}

\begin{lemma}\label{disjunctive-preceq-minimality-of-answer-sets}
Every answer set $M$ of a DLPOD $P$, is a $\preceq$-minimal model of $P$.
\end{lemma}
%

%

\begin{theorem}\label{disjunctive-logical-characterization-theorem}
Let $P$ be a DLPOD. Then, $M$ is an answer set of $P$ iff $M$ is a
consistent \mbox{$\preceq$-minimal} model of $P$ and $M$ is solid.
\end{theorem}

The similarity of the definitions and of the theoretical results of DLPODs to
those of standard LPODs, makes us believe that this is indeed an interesting
class of programs that deserves further attention.

\section{Related and Future Work}\label{related-work}
The work on LPODs is closely related to \qemph{Qualitative Choice Logic}
(QCL)~\cite{qcl}. QCL is an extension of propositional logic with the
preferential connective ``$\times$'', which has the same intuitive meaning as in
LPODs: $A \times B$ is read \qemph{if possible $A$, but if $A$ is impossible
then at least $B$}. Essentially, QCL is the propositional logic underlying
LPODs. It is worth noting that the semantics of QCL is based on the ``degree of
satisfaction''  of formulas, which is connected to the idea of the degree of
satisfaction of the rules of LPODs (Definition~\ref{degree-of-satisfaction}).
Moreover, as remarked by one of the reviewers of the present paper, the DLPODs
introduced in Section~\ref{disjunctive-lpods} are closely connected to the
``basic choice formulas'' of QCL~\cite[Section 3.1, Definition 8]{qcl}. It would
be interesting to investigate whether our four-valued logic can be used to
provide an alternative semantics for QCL.

The work reported in this paper is closely connected to the work of
\citeN{lpod-cabalar}, who first considered the problem of expressing logically
the semantics of LPODs. The key difference between the two works is that ours
provides a characterization of \emph{both} phases of the production of the most
preferred models of an LPOD, while Cabalar's work concentrates on the first one.

It is important to stress here that both our work as well as the work
of~\citeN{lpod-cabalar}, are influenced by the work of \citeN{Pearce96} who
first gave a logical characterization of the answer sets of extended logic
programs, using \emph{Equilibrium Logic}. This is a non-monotonic logic which is
defined on top of the monotonic logic of \emph{Here-and-There}~\cite{Heyting30},
using a model preference approach. The technique we have proposed in this paper,
when applied to a consistent extended logic program $P$, produces the standard
answer sets of $P$; this is a direct consequence of
Theorem~\ref{logical-characterization-theorem} and
Lemma~\ref{new-answer-sets-coincide}. Therefore, for extended logic programs,
the Equilibrium Logic gives the same outcome as our approach which is based on a
four-valued logic and $\preceq$-minimal models that are \emph{solid}. We believe
that a further investigation of the connections of our approach with that of
Equilibrium Logic is a worthwhile topic.

Our work is the first to provide a purely model-theoretic characterization of the
semantics of LPODs. To our knowledge, the four-valued logic we have
utilized does not appear to be a well-known variant/extension of Here-and-There.
However, some seemingly related logics have been used in the literature of
answer set extensions. The original definition of Equilibrium
Logic included a second constructive negation, which
corresponds to Nelson's strong negation~\cite{Nelson49}. This gave rise to a
five-valued extension of Here-and-There, called ${\cal N}_5$. Also, a logic
called ${\cal X}_5$, that is closely connected to ${\cal N}_5$, was recently
proposed by \citeN{ACF0PV19} in order to capture the semantics of arbitrary
combinations of explicit negation with nested expressions. Both ${\cal N}_5$ and
${\cal X}_5$ appear to be connected to our four-valued logic due to the different
notions of false and true that they employ in order to capture aspects that arise
in answer set semantics. However, the ordering of the truth values and the semantics
of the logical connectives are different, and the exact correspondence (if any)
between these logics and the present one, is not straightforward to establish.
This is certainly an interesting topic for further investigation.

Another promising topic for future work is the characterization of the notion of
\emph{strong equivalence}~\cite{DBLP:journals/tocl/LifschitzPV01} for LPODs and DLPODs.
When two logic programs are strongly equivalent, we can replace one for the other
inside a bigger program without worrying that the semantics of the bigger program
will be affected. Characterizations of strong equivalence for LPODs have already
been obtained by \citeN{FaberTW08}. It would be interesting to investigate if the logical
characterization of the semantics of LPODs and DLPODs developed in the present paper,
can offer advantages compared with their work.

\section*{Acknowledgments}
We would like to thank the three anonymous reviewers of our paper for their
careful and insightful comments.

This research is co-financed by Greece and the European Union (European Social
Fund- ESF) through the Operational Programme ``Human Resources Development,
Education and Lifelong Learning 2014- 2020'' in the context of the project
``Techniques for implementing qualitative preferences in deductive querying
systems'' (5048151).

\bibliography{lpods}

\ifincludeappendix
\appendix
\clearpage

\section{Proofs of Section~\ref{lpod-revised-answersets}}

\begin{relemma}{new-answer-sets-coincide}
Let $P$ be a consistent extended logic program. Then the three-valued answer sets of $P$
coincide with the standard answer sets of $P$.
\end{relemma}

\begin{proof}
By taking $n=1$ in Definition~\ref{lpod-reduct}, we get the standard definition
of reduct for consistent extended logic programs.
\end{proof}

\begin{relemma}{answer-set-model-of-P}
Let $P$ be an LPOD and let $M$ be an answer set of $P$. Then, $M$ is a
model of $P$.
\end{relemma}
\begin{proof}
Consider any rule $R$ in $P$ of the form:
\[
  C_1 \times \cdots \times C_n \leftarrow A_1,\ldots,A_m,{\pnot B_1},\ldots,{\pnot B_k}
\]
If $R_{\times}^M = \emptyset$, then $M(B_i) = T$ for some $i$, $1 \leq i \leq k$.
But then, the body of the rule $R$ evaluates to $F$ under $M$, and therefore
$M$ satisfies $R$. Consider now the case where $R_{\times}^M$ is nonempty and
consists of the following rules:
\[
  \begin{array}{lll}
    C_1     & \leftarrow & F^*,A_1,\ldots,A_m\\
            & \cdots     &    \\
    C_{r-1} & \leftarrow & F^*,A_1,\ldots,A_m\\
    C_r     & \leftarrow & A_1,\ldots,A_m
  \end{array}
\]
We distinguish cases based on the value of $M(A_1,\ldots,A_m)$:

\vspace{0.2cm}
\noindent
{\em Case 1:} $M(A_1,\ldots,A_m) = F$.
Then, for some $i$, $M(A_i) = F$. Then, rule $R$ is trivially satisfied by $M$.

\vspace{0.2cm}
\noindent
{\em Case 2:} $M(A_1,\ldots,A_m) = F^*$.
This implies that $M(C_r) \geq F^*$. We distinguish two subcases. If $r=n$ then
$M({C_1 \times \cdots \times C_n}) = M({C_1 \times \cdots \times C_r}) \geq F^*$
because, by the definition of $P_{\times}^M$ it is $M(C_1) = \cdots =
M(C_{r-1})= F^*$ and we also know that $M(C_r) \geq F^*$. Thus, in this subcase
$M$ satisfies $R$. If $r<n$, then by the definition of $P_{\times}^M$, $M(C_r)
\neq F^*$; however, we know that $M(C_r) \geq F^*$, and thus $M(C_r) = T$. Thus,
in this subcase $M$ also satisfies $R$.

\vspace{0.2cm}
\noindent
{\em Case 3:} $M(A_1,\ldots,A_m) = T$.
Then, for all $i$, $M(A_i) = T$. Since $M$ is a model of $P_{\times}^M$, we have
$M(C_r) = T$. Moreover, by the definition of $P_{\times}^M$, $M(C_1) = \cdots =
M(C_{r-1})= F^*$. This implies that $M(C_1 \times \cdots \times C_n) = T$.
\end{proof}

\begin{relemma}{model-satisfies-reduct}
Let $M$ be a model of an LPOD $P$. Then, $M$ is a model of $P^M_{\times}$.
\end{relemma}
\begin{proof}
Consider any rule $R$ in $P$ of the form:
\[
  C_1 \times \cdots \times C_n \leftarrow A_1,\ldots,A_m,{\pnot B_1},\ldots,{\pnot B_k}
\]
and assume $M$ satisfies $R$. If $M(B_i) = T$ for some $i$, $1\leq i \leq k$,
then no rule is created in $P^M_{\times}$ for $R$. Assume therefore that
$M({\pnot B_1},\ldots,{\pnot B_k}) = T$. By the definition of $P^M_{\times}$ the
following rules have been added to $P^M_{\times}$:
\[
\begin{array}{lll}
          C_1     & \leftarrow & F^*,A_1,\ldots,A_m\\
                  & \cdots     &    \\
          C_{r-1} & \leftarrow & F^*,A_1,\ldots,A_m\\
          C_r     & \leftarrow & A_1,\ldots,A_m
\end{array}
\]
where $r$ is the least index such that $M(C_1) = \cdots = M(C_{r-1})= F^*$
and either $r=n$ or $M(C_r) \neq F^*$. Obviously, the first $r-1$ rules above
are satisfied by $M$. For the rule $C_r \leftarrow A_1,\ldots,A_m$ we distinguish
two cases based on the value of $M(A_1,\ldots,A_m)$. If $M(A_1,\ldots,A_m) = F$,
then, the rule is trivially satisfied. If $M(A_1,\ldots,A_m) > F$, then, since
rule $R$ is satisfied by $M$ and $M(C_r) \neq F^*$, it has to be $M(C_r)=T$.
Therefore, the rule $C_r \leftarrow A_1,\ldots,A_m$ is satisfied by $M$.
\end{proof}
\begin{relemma}{preceq-minimality-of-answer-sets}
Every (three-valued) answer set $M$ of an LPOD $P$, is a \mbox{$\preceq$-minimal} model of $P$.
\end{relemma}
\begin{proof}
Assume there exists a model $N$ of $P$ with $N \preceq M$. We will show that $N$
is also a model of $P_{\times}^M$. Since $N \preceq M$, we also have $N \leq M$.
Since $M$ is the $\leq$-least model of $P_{\times}^M$, we will conclude that
$N = M$.

Consider any rule $R$ in $P$ of the form:
\[
  C_1 \times \cdots \times C_n \leftarrow A_1,\ldots,A_m,{\pnot B_1},\ldots,{\pnot B_k}
\]
Assume that $R_{\times}^M$ is nonempty. This means that there exists some $r$,
$1 \leq r \leq n$, such that $M(C_1) = \cdots = M(C_{r-1})= F^*$ and either
$r=n$ or $M(C_r) \neq F^*$. Then, $R_{\times}^M$ consists of the following
rules:
\[
  \begin{array}{lll}
    C_1     & \leftarrow & F^*,A_1,\ldots,A_m\\
            & \cdots     &    \\
    C_{r-1} & \leftarrow & F^*,A_1,\ldots,A_m\\
    C_r     & \leftarrow & A_1,\ldots,A_m
  \end{array}
\]
We show that $N$ satisfies the above rules. We distinguish cases based on the
value of $M(A_1,\ldots,A_m)$:

\vspace{0.2cm}
\noindent
{\em Case 1:} $M(A_1,\ldots,A_m) = F$.
Then, $N(A_1,\ldots,A_m) = F$ and the above rules are trivially satisfied by
$N$.

\vspace{0.2cm}
\noindent
{\em Case 2:} $M(A_1,\ldots,A_m) = F^*$.
Then, since $N\preceq M$, it is $N(A_1,\ldots,A_m) \leq  F^*$. If
$N(A_1,\ldots,A_m)= F$ then $N$ trivially satisfies all the above rules. Assume
therefore that $N(A_1,\ldots,A_m)=  F^*$. Recall now that $M(C_i) = F^*$ for all
$i$, $1\leq i \leq r-1$. Moreover, it has to be $M(C_r) \geq F^*$, because
otherwise $M$ would not satisfy the rule $R$. Since $N \preceq M$, it can only
be $N(C_i) = F^*$ for all $i$, $1\leq i \leq r-1$ and $N(C_r) \geq F^*$, because
otherwise $N$ would not be a model of $P$. Therefore, $N$ satisfies the given
rules of $P^M_{\times}$.

\vspace{0.2cm}
\noindent
{\em Case 3:} $M(A_1,\ldots,A_m) = T$.
Then, since $N\preceq M$, it is either $N(A_1,\ldots,A_m)=F$ or $N(A_1,\ldots,A_m)=T$.
If $N(A_1,\ldots,A_m)= F$ then $N$ trivially satisfies all the above rules. Assume
therefore that $N(A_1,\ldots,A_m)=T$. Recall now that $M(C_i) = F^*$ for all
$i$, $1\leq i \leq r-1$. Moreover, it has to be $M(C_r)= T$, because otherwise
$M$ would not satisfy the rule $R$. Since $N \preceq M$, it can only be $N(C_i) = F^*$
for all $i$, $1\leq i \leq r-1$ and $N(C_r)=T$, because otherwise $N$ would
not be a model of $P$. Therefore, $N$ satisfies the given rules of $P^M_{\times}$.
\end{proof}

In the proofs that follow, we will use the term \emph{Brewka-model} to refer to
that of Definition~\ref{brewka-model} and \emph{Brewka-reduct} to refer to that
of Definition~\ref{brewka-reduct} (although, to be precise, this definition of
reduct was initially introduced in the paper by \citeN{lpod-BNS04}).

In order to establish Lemmas~\ref{collapse-is-brewka-answerset}
and~\ref{brewka-answerset-is-collapse} we first show the following three
propositions.
\begin{proposition}\label{model-is-also-brewka-model}
Let $P$ be an LPOD and let $M$ be a three-valued model of $P$. Then,
$N = \mathit{collapse}(M)$ is a Brewka-model of $P$.
\end{proposition}
\begin{proof}
Consider any rule $R$ of $P$ of the form
\[
  C_1 \times \cdots \times C_n \leftarrow A_1,\ldots,A_m,{\pnot B_1},\ldots,{\pnot B_k}
\]
If there exists $A_i \not\in N$ or there exists $B_j \in N$ then then $N$ trivially
satisfies $R$. Assume that $\{ A_1, \ldots, A_m \} \subseteq N$ and
$\{ B_1, \ldots, B_k \}\cap N = \emptyset$. By Definition~\ref{collapse-def}
it follows that $M(A_1, \ldots, A_m, \pnot B_1, \ldots, \pnot B_k) = T$.
Since $M$ is a three-valued model of $P$, it must satisfy $R$ and therefore
$M(C_1 \times \cdots \times C_n) = T$. Then, there exists $r \leq n$ such
that $M(C_r) = T$ and by Definition~\ref{collapse-def} we get that $C_r \in N$.
Therefore, $N$ satisfies rule $R$.
\end{proof}

\begin{proposition}\label{brewka-model-of-reduct}
Let $P$ be an LPOD and $M$ be a Brewka-model of $P$. Then, $M$ is also a model
of the Brewka-reduct $P^M_{\times}$.
\end{proposition}
\begin{proof}
Consider any rule $R$ in $P$ of the form:
\[
  C_1 \times \cdots \times C_n \leftarrow A_1,\ldots,A_m,{\pnot B_1},\ldots,{\pnot B_k}
\]
and assume $M$ satisfies $R$. If there exists $B_i \in M$ for some $1\leq i \leq k$,
then no rule is created in the Brewka-reduct for $R$. Moreover, if for all
$i \leq n$, $C_i \not\in M$ then also no rule is created in the Brewka-reduct.
Assume therefore that $\{B_1,\ldots, B_k\}\cap M = \emptyset$ and there exists
$r \leq n$ such that $C_r \in M$ and $\{ C_1, \ldots, C_{r-1} \} \cap M = \emptyset$.
By the definition of $P^M_{\times}$ the only rule added to $P^M_{\times}$ because of $R$ is
$C_r \leftarrow A_1,\ldots,A_m$. Since $C_r \in M$ the rule is satisfied by $M$.
\end{proof}
\begin{proposition}\label{three-valued-same-collapse}
Let $P$ be an LPOD and let $M_1,M_2$ be three-valued answer sets of $P$ such that
$\mathit{collapse}(M_1)=\mathit{collapse}(M_2)$. Then, $M_1 = M_2$.
\end{proposition}
\begin{proof}
Assume, for the sake of contradiction, that $M_1 \neq M_2$. We define:
\[
    M(A) = \left\{ \begin{array}{ll}
                      M_1(A)  & \mbox{if $M_1(A) = M_2(A)$}\\
                      F        & \text{otherwise}
                    \end{array} \right.
\]
It is $M \prec M_1$ and $M \prec M_2$. We claim that $M$ is a model of $P$. This will lead to
contradiction because, by Lemma~\ref{preceq-minimality-of-answer-sets}, $M_1$ and $M_2$ are
$\preceq$-minimal models of $P$.

Consider any rule $R$ in $P$ of the form:
\[
  C_1 \times \cdots \times C_n \leftarrow A_1,\ldots,A_m,{\pnot B_1},\ldots,{\pnot B_k}
\]
If $M(B_i) = T$ for some $i$, $1 \leq i \leq k$, then $M$ satisfies the rule. Assume therefore
that $M(B_i)\neq T$ for all $i$, $1 \leq i \leq k$. We distinguish cases:

\vspace{0.2cm}
\noindent
{\em Case 1:} $M(A_1,\ldots,A_m) = F$. Then, obviously, $M$ satisfies $R$.

\vspace{0.2cm}
\noindent
{\em Case 2:} $M(A_1,\ldots,A_m) = F^*$. Then, $M_1(A_1,\ldots,A_m) = F^*$ and
$M_2(A_1,\ldots,A_m) = F^*$. Since, by Lemma~\ref{answer-set-model-of-P}, $M_1$
and $M_2$ are models of $P$ it follows that $M_1(C_1 \times \cdots \times C_n)
\geq F^*$ and $M_2(C_1 \times \cdots \times C_n) \geq F^*$. First assume that
$M_1(C_1 \times \cdots \times C_n) =  T$. This implies that there exists $1 \leq
r \leq n$ such that $M_1(C_r) = T$ and $M_1(C_i) = F^*$ for all $1 \leq i < r$.
Since, by assumption $\mathit{collapse}(M_1) = \mathit{collapse}(M_2)$ it
follows that $M_2(C_r) = T$ and therefore $M(C_r) = T$. Moreover, it must be
$M_2(C_i) = F^*$ for all $i < r$ because we have already established that
$M_2(C_1 \times \cdots \times C_n) \geq F^*$. Therefore, $M(C_i) = F^*$ and
$M(C_1 \times \cdots \times C_n) = T$ and $M$ satisfies the rule. Now assume
that $M_1(C_1 \times \cdots \times C_n) = F^*$. It is easy to see that the only
case is $M_1(C_i) = F^*$ for all $1 \leq i \leq n$. Since $M_2$ has the same
collapse with $M_1$ it follows that $M_2(C_i) \leq F^*$ and because $M_2(C_1
\times \cdots \times C_n) \geq F^*$ it also follows that $M_2(C_i) = F^*$. By
definition of $M$, $M(C_i) = F^*$ for all $1 \leq i \leq n$ and
$M(C_1 \times \cdots \times C_n) = F^*$.

\vspace{0.2cm}
\noindent
{\em Case 3:} $M(A_1,\ldots,A_m) = T$. Then, $M_1(A_1,\ldots,A_m) = T$ and
$M_2(A_1,\ldots,A_m) = T$ and therefore $M_1(C_1 \times \cdots \times C_n) = T$
and $M_2(C_1 \times \cdots \times C_n) = T$. This implies that there exists $r$
such that $M_1(C_1) = M_2(C_1) = F^*,\ldots,M_1(C_{r-1}) = M_2(C_{r-1}) = F^*$,
and $M_1(C_{r}) = M_2(C_{r}) = T$. Therefore, $M(C_1)=\cdots = M(C_{r-1}) = F^*$
and $M(C_r)=T$, which implies that $M(C_1 \times \cdots \times C_n)=T$, and
therefore $M$ satisfies $R$.
\end{proof}

\begin{relemma}{collapse-is-brewka-answerset}
Let $P$ be an LPOD and $M$ be a three-valued answer set of $P$. Then,
$\mathit{collapse}(M)$ is an answer set of $P$ according to
Definition~\ref{brewka-answerset}.
\end{relemma}

\begin{proof}
Since $M$ is an answer set of $P$, then by Lemma~\ref{answer-set-model-of-P}, $M$
is also a model of $P$. Moreover, by Proposition~\ref{model-is-also-brewka-model},
$N = \mathit{collapse(M)}$ is a Brewka-model of $P$. It also follows from
Proposition~\ref{brewka-model-of-reduct} that $N$ is a model of the Brewka-reduct $P^N$.
It suffices to show that $N$ is also the minimum model of $P^N$. Assume there exists
$N'$ that is a model of $P^N$ and $N' \subset N$. We define $M'$ as
\[
    M'(A) = \left\{ \begin{array}{ll}
                      F^*  & A \in N \text{ and } A \not\in N' \\
                      M(A) & \text{otherwise}
                    \end{array} \right.
\]
It is easy to see that $M' < M$. We will show that $M'$ is also model of $P^M_\times$
leading to contradiction because we assume that $M$ is the minimum model of $P^M_\times$.
Consider first a rule of the form $C_i \leftarrow F^*, A_1, \ldots, A_m$.
Since $M$ is an answer set of $P$ it must be $M(C_i) = F^*$. By the definition
of $M'$ it follows that $M'(C_i) \geq F^*$ and $M'$ satisfies the rule.
Now consider a rule of the form $C_r \leftarrow A_1, \ldots, A_m$. We
distinguish cases based on the value of $M(A_1, \ldots, A_m)$.

\vspace{0.2cm}
\noindent
\emph{Case 1:} $M(A_1, \ldots, A_m) = F$. Then, since $M' < M$ it is
$M'(A_1, \ldots, A_m) = F$ and the rule is trivially satisfied.

\vspace{0.2cm}
\noindent
\emph{Case 2:} $M(A_1, \ldots, A_m) = F^*$. Then, $M(A_i) \geq F^*$ and there exists
$A_i$ such that $M(A_i) = F^*$. It follows that $A_i \not\in N$ and therefore
$M'(A_i) = M(A_i) = F^*$. Moreover, by the construction of $M'$, for all $A_i$
we have $M'(A_i) \geq F^*$ and therefore $M'(A_1, \ldots, A_m) = F^*$.
Since $M$ is a model of $P^M_\times$, $M(C_r) \geq F^*$. Again, by the construction
of $M'$ we have $M'(C_r) \geq F^*$ and the rule is satisfied.

\vspace{0.2cm}
\noindent
\emph{Case 3:} $M(A_1, \ldots, A_m) = T$. By the construction of $P^M_\times$ the
rule $C_r \leftarrow A_1, \ldots, A_m$ is a result of a rule in $P$ of the form
\[
  C_1 \times \cdots \times C_r \times \cdots \times C_n \leftarrow A_1,\ldots,A_m,{\pnot B_1},\ldots,{\pnot B_k}
\]
and it must be $M(C_i) = F^*$ for all $i \leq r-1$ and $M(B_j) \leq F^*$ for all
$1 \leq j \leq k$. It follows that $\{C_1, \ldots, C_{r-1}\}\cap N = \emptyset$
and $\{B_1, \ldots, B_k\}\cap N = \emptyset$. Moreover, since $M$ is a model of
$P^M_\times$ we get that $M(C_r) = T$ and it follows that $C_r \in N$. By the
construction of the Brewka-reduct, there exists a rule
$C_r \leftarrow A_1, \ldots, A_m$ in $P^N$. We distinguish two cases.
If $\{ A_1, \ldots, A_m \}\subseteq N'$ then $C_r \in N'$ because $N'$ is a
model of $P^N$. It follows by the construction of $M'$ that
$M'(C_r) = M(C_r) = T$ and $M'$ satisfies the rule.
Otherwise, there exists $l$, $1 \leq l \leq m$ such that $A_l \not\in N'$. Notice also that
$\{ A_1, \ldots, A_m \} \subseteq N$, so $A_l \in N$. Therefore, $M'(A_l) = F^*$
and $M'(A_1, \ldots, A_m) \leq F^*$. Moreover, since $C_r \in N$, we have $M'(C_r) \geq F^*$
that satisfies the rule.
\end{proof}

\begin{relemma}{brewka-answerset-is-collapse}
Let $N$ be an answer set of $P$ according to Definition~\ref{brewka-answerset}.
There exists a unique three-valued interpretation $M$ such that $N =
\mathit{collapse}(M)$ and $M$ is a three-valued answer set of $P$.
\end{relemma}

\begin{proof}
We construct iteratively a set of literals that must have the value $F^*$ in $M$.
Let ${\cal F}^n$ be the sequence:
\begin{alignat*}{2}
  {\cal F}^0 &=  \emptyset \\
  {\cal F}^{n+1} &=  \{ C_j \mid && (C_1 \times \cdots \times C_n \leftarrow A_1, \ldots, A_m, \pnot B_1, \ldots \pnot B_k) \in P \\
             &               && \text{ and } \{ B_1, \ldots, B_k \} \cap N = \emptyset \\
             &               && \text{ and } \{ C_1, \ldots, C_j \} \cap N = \emptyset \\
             &               && \text{ and } \{ A_1, \ldots, A_m \} \subseteq N \cup {\cal F}^n \} \\
  {\cal F}^\omega &=   \rlap{$\cup_{n < \omega}{\cal F}^n$}
\end{alignat*}
We construct $M$ as
\[
    M(A) = \left\{  \begin{array}{ll}
                      F    & A \not\in N \text{ and } A \not\in {\cal F}^\omega \\
                      F^*  & A \not\in N \text{ and } A \in {\cal F}^\omega \\
                      T    & A \in N
                    \end{array} \right.
\]
First we prove that $M$ is a model of $P^M_\times$. Consider first any rule of
the form $C_i \leftarrow F^*, A_1, \ldots, A_m$. By the construction of
$P^M_\times$, such a rule exists because $M(C_i) = F^*$; therefore $M$
satisfies this rule.
Now consider any rule of the form $C_r \leftarrow A_1, \ldots, A_m$. Such a
rule was produced by a rule $R$ in $P$ of the form
\[
  C_1 \times \cdots \times C_r \times \cdots \times C_n \leftarrow
                      A_1, \ldots, A_n, \ldots, \pnot B_1, \ldots, \pnot B_k.
\]
By the construction of $P^M_\times$ it follows that $M(C_i) = F^*$ for all $i < r$.
Therefore $C_i \not\in N$ and also $C_i \in {\cal F}^\omega$ for all $i < r$.
Moreover, it must be $M(B_j) \leq F^*$ for all $1 \leq j \leq k$, so
$\{ B_1, \ldots, B_k \}\cap N = \emptyset$. We distinguish cases based on
the value of $M(A_1,\ldots, A_m)$.

\vspace{0.2cm}
\noindent
\emph{Case 1:}
If $M(A_1,\ldots, A_m) = F$ then the rule is trivially satisfied by $M$.

\vspace{0.2cm}
\noindent
\emph{Case 2:}
If $M(A_1, \ldots, A_m) = F^*$ then for some $A_i$, $M(A_i) = F^*$.
By the construction of $M$, it follows that $A_i \in {\cal F}^\omega$.
It follows by the definition of ${\cal F}^\omega$
that $C_r \in {\cal F}^\omega$ and therefore $M(C_r) \geq F^*$.

\vspace{0.2cm}
\noindent
\emph{Case 3:}
If $M(A_1, \ldots, A_m) = T$ then $\{A_1, \ldots, A_m\} \subseteq N$ and since
$N$ is an answer set according to Definition~\ref{brewka-answerset} it follows
that $N$ is a model of $P$. It follows that there exists a least $j \leq n$
such that $C_j \in N$. Since we have already established that for all $i < r$,
$C_i \not\in N$ it must be $r \leq j \leq n$. But, if $r < j$ then
$C_r \not\in N$ and by the construction of $M$ it must be $M(C_r) = F^*$.
If $M(C_r) = F^*$, then, by the construction of $P^M_\times$, the rule for $C_r$
should be of the form $C_r \leftarrow F^*, A_1, \ldots, A_m$. So, it must
$j = r$ and $C_r \in N$. Therefore, $M(C_r) = T$ and $M$ satisfies the rule.

Therefore, we have established that $M$ is a model of $P^{M}_{\times}$. It remains
to show that $M$ is the $\leq$-least model of $P^{M}_{\times}$. Assume now that there
exists $M'$ that is a model of $P^M_\times$ and $M' < M$.
Let $N' = \mathit{collapse}(M')$. We distinguish two cases.

\vspace{0.2cm}
\noindent
\emph{Case 1:} $N' = N$ and thus $M'$ differs from $M$ only on some atoms $C_r$
such that $M'(C_r) = F$ and $M(C_r) = F^*$. First, by the construction of $M$,
if $M(C_r) = F^*$ then $C_r \in {\cal F}^\omega$. We show by induction on $n$
that for every $C_r \in {\cal F}^n$, $M'(C_r) \geq F^*$. This leads to
contradiction and therefore $M$ is minimal.

\vspace{0.2cm}
\noindent
\emph{Induction base: $n=0$}: the statement is satisfied vacuously.

\vspace{0.2cm}
\noindent
\emph{Induction step: $n=n_0+1$}:
Every atom $C_r\in {\cal F}^{n_0+1}$ must occur in a head of a rule in $P$.
such that $\{C_1, \ldots, C_{r-1}\}\cap N = \emptyset$ and therefore
$\{ C_1, \ldots, C_r \} \subseteq {\cal F}^{n_0+1}$. It follows then
that $M(C_i) = F^*$ for $1 \leq i \leq r$. By the
construction of $P^M_\times$, for every atom $C_r \in {\cal F}^{n_0+1}$
there must be a rule in $P^M_\times$ either of the form
$C_r \leftarrow F^*, A_1, \ldots, A_m$ or of the form
$C_r \leftarrow A_1, \ldots, A_m$. Moreover, since $C_r \in {\cal F}^{n_0+1}$
it follows that
$\{ A_1, \ldots, A_m \} \subseteq N \cup {\cal F}^{n_0}$.
Therefore, by the induction hypothesis,
$M(A_1, \ldots, A_m) = M'(A_1, \ldots, A_m) \geq F^*$. Since $M'$ is also
a model of $P^M_\times$ it must satisfy those rules thus $M'(C_r) \geq F^*$.

\vspace{0.2cm}
\noindent
\emph{Case 2:} $N' \subset N$. We show that $N'$ is a model of $P^N$ leading to
contradiction because, by definition, $N$ is the minimum model of $P^N$.
Consider a rule $R$ of the form $C_r \leftarrow A_1, \ldots, A_m$ in $P^N$.
The rule $R$ has been produced by a rule in $P$ of the form:
\[
  C_1 \times \cdots \times C_r \times \cdots \times C_n \leftarrow
                                A_1,\ldots,A_m,{\pnot B_1},\ldots,{\pnot B_k}
\]
such that $\{ C_1, \ldots, C_{r-1} \}\cap N = \emptyset$ and $C_r \in N$.

If there exists $A_i \not\in N$ then also $A_i \not\in N'$ and the rule is
trivially satisfied by $N'$. Assume, on the other hand, that
$\{ A_1,\ldots, A_n \} \subseteq N$. It follows, by the
definition of $M$, that $M(A_1, \ldots, A_m) = T$, $M(C_i) = F^*$ for $i < r$
and $M(C_r) = T$. Therefore, there exist a rule in $P^M_\times$ of the form $C_r
\leftarrow A_1, \ldots, A_m$.
If $M'(A_1, \ldots, A_m) = F$ or $M'(A_1, \ldots, A_m) = F^*$ then there exists
$A_i \not\in N'$ and $N'$ again satisfies the rule.
If $M'(A_1, \ldots, A_m) = T$ then since $M'$ is a model of $P^M_\times$ it
follows that $M'(C_r) = T$. Since $N'$ is the collapse of $M'$ it is
$\{A_1, \ldots, A_m\} \subseteq N'$ and $C_r \in N'$. Therefore, $N'$ satisfies
the rule $R$ in $P^N$.

The uniqueness of $M$ follows directly from Proposition~\ref{three-valued-same-collapse}.
\end{proof}

\section{Proofs of Section~\ref{lpod-logical-characterization}}
In order to establish Theorem~\ref{logical-characterization-theorem}, we show two lemmas
(which essentially establish the left-to-right and the right-to-left directions of the
theorem, respectively).
\begin{lemma}\label{answer-set-implies-minimality}
Let $P$ be an LPOD program and let $M$ be an answer set of $P$. Then, $M$ is a
$\preceq$-minimal model of $P$ and $M$ is solid.
\end{lemma}
\begin{proof}
Since $M$ is an answer set of $P$, then, by Lemma~\ref{answer-set-model-of-P},
$M$ is a model of $P$. Moreover, $M$ is solid because our definition of answer
sets does not involve the value $T^*$. It remains to show that it is minimal
with respect to the $\preceq$~ordering. Assume, for the sake of contradiction,
that there exists a model $N$ of $P$ with $N \prec M$. By
Lemma~\ref{preceq-minimality-of-answer-sets}, $M$ is (three-valued)
$\preceq$-minimal. Therefore, $N$ can not be solid. We first show that $N$ can
not be a model of the reduct $P^{M}_{\times}$. Assume for the sake of
contradiction that $N$ is a model of $P^{M}_{\times}$. We construct the
following interpretation $N'$:
\[
    N'(A)  = \left\{ \begin{array}{ll}
                        F^*,  & \mbox{if $N(A)=T^*$} \\
                        N(A), & \mbox{otherwise}
                     \end{array}\right.
\]
We claim that $N'$ must also be a model of $P^{M}_{\times}$. Consider first a
rule of the form $C \leftarrow F^*,A_1,\ldots,A_m$. Since $N$ is a model of
$P^{M}_{\times}$, it is $N(C) \geq F^*$. By the definition of $N'$, it is
$N(C) \geq F^*$ and therefore $N'$ satisfies this rule. Consider now a rule of
the form $C \leftarrow A_1,\ldots,A_m$ in $P^{M}_{\times}$. We show that $N'$
also satisfies this rule. We perform a case analysis:

\vspace{0.2cm}
\noindent
{\em Case 1:} $N(A_1,\ldots,A_m) = F$. Then, $N'(A_1,\ldots,A_m) = F$ and $N'$
trivially satisfies the rule.

\vspace{0.2cm}
\noindent
{\em Case 2:} $N(A_1,\ldots,A_m) = F^*$. Then, $N'(A_1,\ldots,A_m) = F^*$.
Moreover, $N(C)\geq F^*$ because $N$ is a model of $P^{M}_{\times}$. By the
definition of $N'$, it is $N'(C) \geq F^*$, and therefore $N'$ satisfies the
rule.

\vspace{0.2cm}
\noindent
{\em Case 3:} $N(A_1,\ldots,A_m) = T^*$. Then, $N'(A_1,\ldots,A_m) = F^*$.
Moreover, $N(C)\geq T^*$ because $N$ is a model of $P^{M}_{\times}$. By the
definition of $N'$, it is $N'(C) \geq F^*$, and therefore $N'$ satisfies the
rule.

\vspace{0.2cm}
\noindent
{\em Case 4:} $N(A_1,\ldots,A_m) = T$. Then, $N'(A_1,\ldots,A_m) = T$. Moreover,
$N(C)=T$ because $N$ is a model of $P^{M}_{\times}$. By the definition of $N'$,
it is $N'(C)=T$, and therefore $N'$ satisfies the rule.

\vspace{0.2cm}
\noindent
Therefore, $N'$ must also be a model of $P^{M}_{\times}$. Moreover, by
definition, $N'$ is solid and $N' < M$. This contradicts the fact that, by
construction, $M$ is the $\leq$-least model of $P^{M}_{\times}$. In conclusion,
$N$ can not be a model of $P^{M}_{\times}$.


We now show that $N$ can not be a model of $P$. As we showed above, $N$ is not a
model of $P^{M}_{\times}$, and consequently there exists a rule in
$P^{M}_{\times}$ that is not satisfied by $N$. Such a rule in $P^{M}_{\times}$
must have resulted due to a rule $R$ of the following form in $P$:
\[
  C_1 \times \cdots \times C_n \leftarrow A_1,\ldots,A_m,{\pnot B_1},\ldots,{\pnot B_k}
\]
According to the definition of $P^{M}_{\times}$, for all $i$, $1\leq i \leq k$,
$M(\pnot B_i) = T$, and since $N \prec M$, it is also $N(\pnot B_i) = T$.
Moreover, there exists some $r\leq n$ such that $M(C_1) = \cdots =
M(C_{r-1})= F^*$ and either $r=n$ or $M(C_r) \neq F^*$. Since $N \prec M$, it is
$N(C_i) \leq F^*$ for all $i$, $1\leq i \leq r-1$. Consider now the rule that is
not satisfied by $N$ in $P^{M}_{\times}$. If it is of the form $C_i \leftarrow
F^*,A_1,\ldots,A_m$, $i$, $1\leq i \leq r-1$, then $N(A_1,\ldots,A_m) > F$ and
$N(C_i) = F$. This implies that $N(C_1 \times \cdots \times C_n) =F$ and
therefore $N$ does not satisfy the rule $R$. If the rule that is not satisfied
by $N$ in $P^{M}_{\times}$ is of the form $C_r \leftarrow A_1,\ldots,A_m$, then
$N(C_r) < N(A_1,\ldots,A_m)$ and therefore, since $N(C_i) \leq F^*$ for all $i$,
$1\leq i \leq r-1$, it is:
\[
  N(C_1 \times \cdots \times C_n) < N(A_1,\ldots,A_m,{\pnot B_1},\ldots,{\pnot B_k})
\]
Thus, $N$ is not a model of $P$.
\end{proof}

\begin{lemma}\label{minimality-implies-answerset}
Let $P$ be an LPOD program and let $M$ be a $\preceq$-minimal model of $P$ and
$M$ is solid. Then, $M$ is an answer set of $P$.
\end{lemma}
\begin{proof}
First observe that, by Lemma~\ref{model-satisfies-reduct}, $M$ is also a model
of $P^{M}_{\times}$. We demonstrate that $M$ is actually the $\leq$-least model
of $P^{M}_{\times}$. Assume, for the sake of contradiction, that $N$ is the
$\leq$-least model of $P^{M}_{\times}$. Then, $N$ will differ from $M$ in some
atoms $A$ such that $N(A) < M(A)$. We distinguish two cases. In the first case
all the atoms $A$ such that $N(A) < M(A)$ have $M(A) \leq F^*$. In the second
case there exist at least one atom $A$ such that $M(A) > F^*$.

In the first case it is easy to see that $N \prec M$. We demonstrate that $N$ is
also model of $P$ leading to contradiction since $M$ is $\preceq$-minimal.
Assume that $N$ is not a model of $P$. Then, there exists in $P$ a rule $R$ of
the form:
\[
  C_1 \times \cdots \times C_n \leftarrow A_1,\ldots,A_m,{\pnot B_1},\ldots,{\pnot B_k}
\]
such that $N(C_1 \times \cdots \times C_n) < N'(A_1,\ldots,A_m,{\pnot B_1},
\ldots,{\pnot B_k})$. Notice that this implies that $N({\pnot B_1},\ldots,{\pnot B_k})
= M({\pnot B_1},\ldots,{\pnot B_k}) = T$. Therefore,
$N({C_1 \times \cdots \times C_n}) < N(A_1,\ldots,A_m)$.
We distinguish cases based on the value of $N(A_1,\ldots,A_m)$:

\vspace{0.2cm}
\noindent
{\em Case 1:} $N(A_1,\ldots,A_m) = F$. This case leads immediately to
contradiction because $N$ trivially satisfies $R$.

\vspace{0.2cm}
\noindent
{\em Case 2:} $N(A_1,\ldots,A_m) > F$.
Then, $N(A_1,\ldots,A_m) = M(A_1,\ldots,A_m)$. Since $M$ is a model of $P$, it
is $M(C_1 \times \cdots \times C_n) \geq M(A_1,\ldots,A_m) > F$. This implies
that there exists some $r$, $1 \leq r \leq n$, such that $M(C_1) = \cdots =
M(C_{r-1}) = F^*$ and $M(C_r) \geq F^*$. By the definition of the reduct, the
rule $C_r \leftarrow A_1,\ldots,A_m$ exists in $P^{M}_{\times}$. Since $N$ is a
model of $P^{M}_{\times}$, we get that $N(C_r) > F$. Moreover, $N$ should also
satisfy the rules $C_i \leftarrow F^*, A_1,\ldots,A_m$ for $1 \leq i \leq r-1$.
Since $N(C_i) \leq M(C_i)$ and $N(C_r) = M(C_r)$ we get $N(C_1)= \cdots =
N(C_{r-1}) = F^*$ and $N(C_r) = M(C_r)$. Therefore $N(C_1 \times \cdots C_n) =
M(C_1 \times \cdots C_n)$ and $N(C_1 \times \cdots C_n) \geq N(A_1,\ldots,A_m)$
(contradiction).

In the second case we construct the following interpretation $N'$:
\[
    N'(A)  =  \left\{ \begin{array}{ll}
                        T^*,  & \mbox{if $M(A)=T$ and $N(A) \in \{F,F^*\}$}\\
                        F^*,  & \mbox{if $M(A) = F^*$}\\
                        N(A), & \mbox{otherwise}
                      \end{array}\right.
\]
It is easy to see that $N'\prec M$. We demonstrate that $N'$ is a model of $P$,
which will lead to a contradiction (since we have assumed that $M$ is
$\preceq$-minimal).

Assume $N'$ is not a model of $P$. Then, there exists in $P$ a rule $R$ of the
form:
\[
  C_1 \times \cdots \times C_n \leftarrow A_1,\ldots,A_m,{\pnot B_1},\ldots,{\pnot B_k}
\]
such that $N'(C_1 \times \cdots \times C_n) < N'(A_1,\ldots,A_m,
{\pnot B_1},\ldots,{\pnot B_k})$. Notice that this implies that
$N'({\pnot B_1},\ldots,{\pnot B_k}) = N({\pnot B_1},\ldots,{\pnot B_k}) =
M({\pnot B_1},\ldots,{\pnot B_k}) = T$. Therefore,
$N'(C_1 \times \cdots \times C_n) < N'(A_1,\ldots,A_m)$. We
distinguish cases based on the value of $N'(A_1,\ldots,A_m)$:
\vspace{0.2cm}
\noindent
{\em Case 1:} $N'(A_1,\ldots,A_m) = F$.
This case leads immediately to contradiction because $N'$ trivially satisfies
$R$.

\vspace{0.2cm}
\noindent
{\em Case 2:} $N'(A_1,\ldots,A_m) = F^*$.
Then, by the definition of $N'$, $M(A_1,\ldots,A_m) = F^*$. Since $M$ is a model
of $P$, it is $M(C_1 \times \cdots \times C_n) \geq F^*$. This implies that
either $M(C_1) = \cdots = M(C_{n}) = F^*$ or there exists $r\leq n$ such that
$M(C_1) = \cdots = M(C_{r-1}) = F^*$ and $M(C_r) = T$. By the definition of
$N'$, we get in both cases $N'(C_1 \times \cdots \times C_n) \geq F^*$
(contradiction).

\vspace{0.2cm}
\noindent
{\em Case 3:} $N'(A_1,\ldots,A_m) = T^*$.
Then, by the definition of $N'$, $M(A_1,\ldots,A_m) = T$. Since $M$ is a model
of $P$, it is $M(C_1 \times \cdots \times C_n) = T$. This implies that there
exists some $r$, $1 \leq r \leq n$, such that $M(C_1) = \cdots = M(C_{r-1}) =
F^*$ and $M(C_r) =T$. By the definition of $N'$, we get that $N'(C_1 \times
\cdots \times C_n) \geq T^*$ (contradiction).

\vspace{0.2cm}
\noindent
{\em Case 4:} $N'(A_1,\ldots,A_m) = T$.
Then, by the definition of $N'$, $N(A_1,\ldots,A_m) = T$ and $M(A_1,\ldots,A_m)
= T$. Since $M$ is a model of $P$, it is $M(C_1 \times \cdots \times C_n) = T$.
This implies that there exists some $r$, $1 \leq r \leq n$, such that $M(C_1) =
\cdots = M(C_{r-1}) = F^*$ and $M(C_r) =T$. By the definition of the reduct, the
rule $C_r \leftarrow A_1,\ldots,A_m$ exists in $P^{M}_{\times}$. Since $N$ is a
model of $P^{M}_{\times}$, we get that $N(C_r)=T$. Thus, $N'(C_1)= \cdots =
N'(C_{r-1}) = F^*$ and $N'(C_r) = T$, and therefore $N'(C_1 \times \cdots \times C_n) =
T$ (contradiction).
\end{proof}

\begin{retheorem}{logical-characterization-theorem}
Let $P$ be an LPOD. Then, $M$ is a three-valued answer set of $P$ iff $M$
is a consistent \mbox{$\preceq$-minimal} model of $P$ and $M$ is solid.
\end{retheorem}
\begin{proof}
Immediate from Lemma~\ref{answer-set-implies-minimality} and Lemma~\ref{minimality-implies-answerset}.
\end{proof}


\section{Proofs of Section~\ref{disjunctive-lpods}}\label{proofs-disjunctive-lpods}
\begin{relemma}{new-disjunctive-answersets-coincide}
Let $P$ be a consistent disjunctive extended logic program. Then, the answer sets of $P$
according to Definition~\ref{disjunctive-answer-set}, coincide with the
standard answer sets of $P$.
\end{relemma}
\begin{proof}
By taking $n=1$ in Definition~\ref{dlpod-reduct}, we get the standard definition
of reduct for consistent disjunctive extended logic programs.
\end{proof}

\begin{relemma}{disjunctive-answer-set-model-of-P}
Let $P$ be a DLPOD program and let $M$ be an answer set of $P$. Then, $M$ is a
model of $P$.
\end{relemma}
\begin{proof}
Consider any rule $R$ in $P$ of the form:
\[
  {\cal C}_1 \times \cdots \times {\cal C}_n \leftarrow A_1,\ldots,A_m,{\pnot B_1},\ldots,{\pnot B_k}
\]
If $R_{\times}^M = \emptyset$, then $M(B_i) = T$ for some $i$, $1 \leq i \leq
k$. But then, the body of the rule $R$ evaluates to $F$ under $M$, and therefore
$M$ satisfies $R$. Consider now the case where $R_{\times}^M$ is nonempty and
consists of the following rules:
\[
  \begin{array}{lll}
    {\cal C}_1     & \leftarrow & F^*,A_1,\ldots,A_m\\
                   & \cdots     &    \\
    {\cal C}_{r-1} & \leftarrow & F^*,A_1,\ldots,A_m\\
    {\cal C}_r     & \leftarrow & A_1,\ldots,A_m
  \end{array}
\]
We distinguish cases based on the value of $M(A_1,\ldots,A_m)$:

\vspace{0.2cm}
\noindent
{\em Case 1:} $M(A_1,\ldots,A_m) = F$.
Then, for some $i$, $M(A_i) = F$. Then, rule $R$ is trivially satisfied by $M$.

\vspace{0.2cm}
\noindent
{\em Case 2:} $M(A_1,\ldots,A_m) = F^*$.
This implies that $M({\cal C}_r) \geq F^*$. We distinguish two subcases. If
$r=n$ then $M({\cal C}_1 \times \cdots \times {\cal C}_n) = M({\cal C}_1 \times
\cdots \times {\cal C}_r) \geq F^*$ because, by the definition of $P_{\times}^M$
it is $M({\cal C}_1) = \cdots = M({\cal C}_{r-1})= F^*$ and we also know that
$M({\cal C}_r) \geq F^*$. Thus, in this subcase $M$ satisfies $R$. If $r<n$,
then by the definition of $P_{\times}^M$, $M({\cal C}_r) \neq F^*$; however, we
know that $M({\cal C}_r) \geq F^*$, and thus $M({\cal C}_r) = T$. Thus, in this
subcase $M$ also satisfies $R$.

\vspace{0.2cm}
\noindent
{\em Case 3:} $M(A_1,\ldots,A_m) = T$.
Then, for all $i$, $M(A_i) = T$. Since $M$ is a model of $P_{\times}^M$, we have
$M({\cal C}_r) = T$. Moreover, by the definition of $P_{\times}^M$,
$M({\cal C}_1) = \cdots = M({\cal C}_{r-1})= F^*$. This implies that
$M({\cal C}_1 \times \cdots \times {\cal C}_n) = T$.
\end{proof}

\begin{relemma}{disjunctive-model-satisfies-reduct}
Let $M$ be a model of a DLPOD $P$. Then, $M$ is a model of $P^M_{\times}$.
\end{relemma}
\begin{proof}
Consider any rule $R$ in $P$ of the form:
\[
  {\cal C}_1 \times \cdots \times {\cal C}_n \leftarrow A_1,\ldots,A_m,{\pnot B_1},\ldots,{\pnot B_k}
\]
and assume $M$ satisfies $R$. If $M(B_i) = T$ for some $i$, $1\leq i \leq k$,
then no rule is created in $P^M_{\times}$ for $R$. Assume therefore that
$M({\pnot B_1},\ldots,{\pnot B_k}) = T$. By the definition of $P^M_{\times}$ the
following rules have been added to $P^M_{\times}$:
\[
\begin{array}{lll}
          {\cal C}_1     & \leftarrow & F^*,A_1,\ldots,A_m\\
                         & \cdots     &    \\
          {\cal C}_{r-1} & \leftarrow & F^*,A_1,\ldots,A_m\\
          {\cal C}_r     & \leftarrow & A_1,\ldots,A_m
\end{array}
\]
where $r$ is the least index such that $M({\cal C}_1) = \cdots = M({\cal C}_{r-1})= F^*$
and either $r=n$ or $M({\cal C}_r) \neq F^*$. Obviously, the first $r-1$ rules above
are satisfied by $M$. For the rule ${\cal C}_r \leftarrow A_1,\ldots,A_m$ we distinguish
two cases based on the value of $M(A_1,\ldots,A_m)$. If $M(A_1,\ldots,A_m) = F$,
then, the rule is trivially satisfied. If $M(A_1,\ldots,A_m) > F$, then, since
rule $R$ is satisfied by $M$ and $M({\cal C}_r) \neq F^*$, it has to be $M({\cal C}_r)=T$.
Therefore, the rule ${\cal C}_r \leftarrow A_1,\ldots,A_m$ is satisfied by $M$.
\end{proof}

\begin{relemma}{disjunctive-preceq-minimality-of-answer-sets}
Every answer set $M$ of a DLPOD $P$, is a $\preceq$-minimal model of $P$.
\end{relemma}
\begin{proof}
Assume there exists a model $N$ of $P$ with $N \preceq M$. We will show that $N$
is also a model of $P_{\times}^M$. Since $N \preceq M$, we also have $N \leq M$.
Since $M$ is the $\leq$-least model of $P_{\times}^M$, we will conclude that
$N = M$.

Consider any rule $R$ in $P$ of the form:
\[
  {\cal C}_1 \times \cdots \times {\cal C}_n \leftarrow A_1,\ldots,A_m,{\pnot B_1},\ldots,{\pnot B_k}
\]
Assume that $R_{\times}^M$ is nonempty. This means that there exists some $r$,
$1 \leq r \leq n$, such that $M({\cal C}_1) = \cdots = M({\cal C}_{r-1})= F^*$
and either $r=n$ or $M({\cal C}_r) \neq F^*$. Then, $R_{\times}^M$ consists of
the following rules:
\[
  \begin{array}{lll}
    {\cal C}_1     & \leftarrow & F^*,A_1,\ldots,A_m\\
                   & \cdots     &    \\
    {\cal C}_{r-1} & \leftarrow & F^*,A_1,\ldots,A_m\\
    {\cal C}_r     & \leftarrow & A_1,\ldots,A_m
  \end{array}
\]
We show that $N$ satisfies the above rules. We distinguish cases based on the
value of $M(A_1,\ldots,A_m)$:

\vspace{0.2cm}
\noindent
{\em Case 1:} $M(A_1,\ldots,A_m) = F$.
Then, $N(A_1,\ldots,A_m) = F$ and the above rules are trivially satisfied by
$N$.

\vspace{0.2cm}
\noindent
{\em Case 2:} $M(A_1,\ldots,A_m) = F^*$.
Then, since $N\preceq M$, it is $N(A_1,\ldots,A_m) \leq  F^*$. If
$N(A_1,\ldots,A_m)= F$ then $N$ trivially satisfies all the above rules. Assume
therefore that $N(A_1,\ldots,A_m)=  F^*$. Recall now that $M({\cal C}_i) = F^*$
for all $i$, $1\leq i \leq r-1$. Moreover, it has to be $M({\cal C}_r) \geq
F^*$, because otherwise $M$ would not satisfy the rule $R$. Since $N \preceq M$,
it can only be $N({\cal C}_i) = F^*$ for all $i$, $1\leq i \leq r-1$ and
$N({\cal C}_r) \geq F^*$, because otherwise $N$ would not be a model of $P$.
Therefore, $N$ satisfies the given rules of $P^M_{\times}$.

\vspace{0.2cm}
\noindent
{\em Case 3:} $M(A_1,\ldots,A_m) = T$.
Then, since $N\preceq M$, it is either $N(A_1,\ldots,A_m)=F$ or
$N(A_1,\ldots,A_m)=T$. If $N(A_1,\ldots,A_m)= F$ then $N$ trivially satisfies
all the above rules. Assume therefore that $N(A_1,\ldots,A_m)=T$. Recall now
that $M({\cal C}_i) = F^*$ for all $i$, $1\leq i \leq r-1$. Moreover, it has to
be $M({\cal C}_r)= T$, because otherwise $M$ would not satisfy the rule $R$.
Since $N \preceq M$, it can only be $N({\cal C}_i) = F^*$ for all $i$, $1\leq i
\leq r-1$ and $N({\cal C}_r)=T$, because otherwise $N$ would not be a model of
$P$. Therefore, $N$ satisfies the given rules of $P^M_{\times}$.
\end{proof}
\begin{retheorem}{disjunctive-logical-characterization-theorem}
Let $P$ be a DLPOD. Then, $M$ is an answer set of $P$ iff $M$ is a
consistent \mbox{$\preceq$-minimal} model of $P$ and $M$ is solid.
\end{retheorem}
The proof of the above theorem follows directly by the following two lemmas.
\begin{lemma}\label{disjunctive-answerset-implies-minimality}
Let $P$ be an DLPOD and let $M$ be an answer set of $P$. Then, $M$ is a
consistent $\preceq$-minimal model of $P$ and $M$ is solid.
\end{lemma}
\begin{proof}
Since $M$ is an answer set of $P$, then, by
Lemma~\ref{disjunctive-answer-set-model-of-P}, $M$ is a model of $P$. Moreover,
$M$ is solid because our definition of answer sets does not involve the value
$T^*$. It remains to show that it is minimal with respect to the
$\preceq$~ordering. Assume, for the sake of contradiction, that there exists a
model $N$ of $P$ with $N \prec M$. By
Lemma~\ref{disjunctive-preceq-minimality-of-answer-sets}, $M$ is (three-valued)
$\preceq$-minimal. Therefore, $N$ can not be solid. We first show that $N$ can
not be a model of the reduct $P^{M}_{\times}$. Assume for the sake of
contradiction that $N$ is a model of $P^{M}_{\times}$. We construct the
following interpretation $N'$:
\[
    N'(A) = \left\{ \begin{array}{ll}
                      F^*,  & \mbox{if $N(A)=T^*$}\\
                      N(A), & \mbox{otherwise}
                    \end{array}\right.
\]
We claim that $N'$ must also be a model of $P^{M}_{\times}$. Consider first a
rule of the form ${\cal C} \leftarrow F^*,A_1,\ldots,A_m$. Since $N$ is a model
of $P^{M}_{\times}$, it is $N({\cal C}) \geq F^*$. By the definition of $N'$, it
is $N({\cal C}) \geq F^*$ and therefore $N'$ satisfies this rule. Consider now a
rule of the form ${\cal C} \leftarrow A_1,\ldots,A_m$ in $P^{M}_{\times}$. We
show that $N'$ also satisfies this rule. We perform a case analysis:

\vspace{0.2cm}
\noindent
{\em Case 1:} $N(A_1,\ldots,A_m) = F$.
Then, $N'(A_1,\ldots,A_m) = F$ and $N'$ trivially satisfies the rule.

\vspace{0.2cm}
\noindent
{\em Case 2:} $N(A_1,\ldots,A_m) = F^*$.
Then, $N'(A_1,\ldots,A_m) = F^*$. Moreover, $N({\cal C})\geq F^*$ because $N$ is
a model of $P^{M}_{\times}$. By the definition of $N'$, it is $N'({\cal C}) \geq
F^*$, and therefore $N'$ satisfies the rule.

\vspace{0.2cm}
\noindent
{\em Case 3:} $N(A_1,\ldots,A_m) = T^*$.
Then, $N'(A_1,\ldots,A_m) = F^*$. Moreover, $N({\cal C})\geq T^*$ because $N$ is
a model of $P^{M}_{\times}$. By the definition of $N'$, it is $N'({\cal C}) \geq
F^*$, and therefore $N'$ satisfies the rule.

\vspace{0.2cm}
\noindent
{\em Case 4:} $N(A_1,\ldots,A_m) = T$.
Then, $N'(A_1,\ldots,A_m) = T$. Moreover, $N({\cal C}) = T$ because $N$ is a
model of $P^{M}_{\times}$. By the definition of $N'$, it is $N'({\cal C}) = T$,
and therefore $N'$ satisfies the rule.

\vspace{0.2cm}
\noindent
Therefore, $N'$ must also be a model of $P^{M}_{\times}$. Moreover, by
definition, $N'$ is solid and $N' < M$. This contradicts the fact that, by
construction, $M$ is the $\leq$-least model of $P^{M}_{\times}$. In conclusion,
$N$ can not be a model of $P^{M}_{\times}$.

We now show that $N$ can not be a model of $P$. As we showed above, $N$ is not a
model of $P^{M}_{\times}$, and consequently there exists a rule in
$P^{M}_{\times}$ that is not satisfied by $N$. Such a rule in $P^{M}_{\times}$
must have resulted due to a rule $R$ of the following form in $P$:
\[
  {\cal C}_1 \times \cdots \times {\cal C}_n \leftarrow A_1,\ldots,A_m,{\pnot B_1},\ldots,{\pnot B_k}
\]
According to the definition of $P^{M}_{\times}$, for all $i$, $1\leq i \leq k$,
$M({\pnot B_i}) = T$, and since $N \prec M$, it is also $N({\pnot B_i}) = T$.
Moreover, there exists some $r\leq n$ such that
$M({\cal C}_1) = \cdots = M({\cal C}_{r-1})= F^*$ and either $r=n$ or
$M({\cal C}_r) \neq F^*$.
Since $N \prec M$, it is $N({\cal C}_i) \leq F^*$ for all $i$, $1\leq i \leq r-1$.
Consider now the rule that is not satisfied by $N$ in $P^{M}_{\times}$. If
it is of the form ${\cal C}_i \leftarrow F^*,A_1,\ldots,A_m$, $i$, $1\leq i \leq r-1$,
then $N(A_1,\ldots,A_m) > F$ and $N({\cal C}_i) = F$. This implies that
$N({\cal C}_1 \times \cdots \times {\cal C}_n) = F$ and therefore $N$ does not
satisfy the rule $R$. If the rule that is not satisfied by $N$ in
$P^{M}_{\times}$ is of the form ${\cal C}_r \leftarrow A_1,\ldots,A_m$, then
$N({\cal C}_r) < N(A_1,\ldots,A_m)$ and therefore, since $N({\cal C}_i) \leq
F^*$ for all $i$, $1 \leq i \leq r-1$, it is:
\[
  N({\cal C}_1 \times \cdots \times {\cal C}_n) < N(A_1,\ldots,A_m,{\pnot B_1},\ldots,{\pnot B_k})
\]
Thus, $N$ is not a model of $P$.
\end{proof}

\begin{lemma}\label{disjunctive-minimality-implies-answerset}
Let $P$ be an DLPOD and let $M$ be a consistent $\preceq$-minimal model of $P$ and
$M$ is solid. Then, $M$ is an answer set of $P$.
\end{lemma}
\begin{proof}
First observe that, by Lemma~\ref{disjunctive-model-satisfies-reduct}, $M$ is
also a model of $P^{M}_{\times}$. We demonstrate that $M$ is actually the
$\leq$-least model of $P^{M}_{\times}$. Assume, for the sake of contradiction,
that $N$ is the $\leq$-least model of $P^{M}_{\times}$. Then, $N$ will differ
from $M$ in some atoms $A$ such that $N(A) < M(A)$. We distinguish two cases. In
the first case all the atoms $A$ such that $N(A) < M(A)$ have $M(A) \leq F^*$.
In the second case there exist at least one atom $A$ such that $M(A) > F^*$.

In the first case it is easy to see that $N \prec M$. We demonstrate that $N$ is
also model of $P$ leading to contradiction since $M$ is $\preceq$-minimal. Assume
that $N$ is not a model of $P$. Then, there exists in $P$ a rule $R$ of the
form:
\[
  {\cal C}_1 \times \cdots \times {\cal C}_n \leftarrow A_1,\ldots,A_m,{\pnot B_1},\ldots,{\pnot B_k}
\]
such that $N({\cal C}_1 \times \cdots \times {\cal C}_n) < N'(A_1,\ldots,A_m,
{\pnot B_1},\ldots,{\pnot B_k})$. Notice that this implies that $N({\pnot B_1},
\ldots,{\pnot B_k}) = M({\pnot B_1},\ldots,{\pnot B_k}) = T$. Therefore,
$N({\cal C}_1 \times \cdots \times {\cal C}_n) < N(A_1,\ldots,A_m)$. We
distinguish cases based on the value of $N(A_1,\ldots,A_m)$:

\vspace{0.2cm}
\noindent
{\em Case 1:} $N(A_1,\ldots,A_m) = F$.
This case leads immediately to contradiction because $N$ trivially satisfies
$R$.

\vspace{0.2cm}
\noindent
{\em Case 2:} $N(A_1,\ldots,A_m) > F$.
Then, $N(A_1,\ldots,A_m) = M(A_1,\ldots,A_m)$. Since $M$ is a model of $P$, it
is $M({\cal C}_1 \times \cdots \times {\cal C}_n) \geq M(A_1,\ldots,A_m) > F$.
This implies that there exists some $r$, $1 \leq r \leq n$, such that
$M({\cal C}_1) = \cdots = M({\cal C}_{r-1}) = F^*$ and $M({\cal C}_r) \geq F^*$.
By the definition of the reduct, the rule ${\cal C}_r \leftarrow A_1,\ldots,A_m$
exists in $P^{M}_{\times}$. Since $N$ is a model of $P^{M}_{\times}$, we get
that $N({\cal C}_r) > F$. Moreover, $N$ should also satisfy the rules
${\cal C}_i \leftarrow F^*, A_1,\ldots,A_m$ for $1 \leq i \leq r-1$. Since
$N({\cal C}_i) \leq M({\cal C}_i)$ and $N({\cal C}_r) = M({\cal C}_r)$ we get
$N({\cal C}_1)= \cdots = N({\cal C}_{r-1}) = F^*$ and $N({\cal C}_r) = M({\cal C}_r)$.
Therefore $N({\cal C}_1 \times \cdots {\cal C}_n) = M({\cal C}_1 \times \cdots
{\cal C}_n)$ and $N({\cal C}_1 \times \cdots {\cal C}_n) \geq N(A_1,\ldots,A_m)$
(contradiction).

In the second case we construct the following interpretation $N'$:
\[
    N'(A) = \left\{ \begin{array}{ll}
                      T^*,  & \mbox{if $M(A)=T$ and $N(A) \in \{F,F^*\}$}\\
                      F^*,  & \mbox{if $M(A) = F^*$}\\
                      N(A), & \mbox{otherwise}
                    \end{array}\right.
\]
It is easy to see that $N'\prec M$. We demonstrate that $N'$ is a model of $P$,
which will lead to a contradiction (since we have assumed that $M$ is
$\preceq$-minimal).

Assume $N'$ is not a model of $P$. Then, there exists in $P$ a rule $R$ of the
form:
\[
  {\cal C}_1 \times \cdots \times {\cal C}_n \leftarrow A_1,\ldots,A_m,{\pnot B_1},\ldots,{\pnot B_k}
\]
such that
$N'({\cal C}_1 \times \cdots \times {\cal C}_n) < N'(A_1,\ldots,A_m,\pnot B_1,\ldots,\pnot B_k)$.
Notice that this implies that
$N'({\pnot B_1},\ldots,{\pnot B_k}) = N({\pnot B_1},\ldots,{\pnot B_k}) =
M({\pnot B_1},\ldots,{\pnot B_k}) = T$. Therefore,
$N'({\cal C}_1 \times \cdots \times {\cal C}_n) < N'(A_1,\ldots,A_m)$.
We distinguish cases based on the value of $N'(A_1,\ldots,A_m)$:

\vspace{0.2cm}
\noindent
{\em Case 1:} $N'(A_1,\ldots,A_m) = F$.
This case leads immediately to contradiction because $N'$ trivially satisfies
$R$.

\vspace{0.2cm}
\noindent
{\em Case 2:} $N'(A_1,\ldots,A_m) = F^*$.
Then, by the definition of $N'$, $M(A_1,\ldots,A_m) = F^*$. Since $M$ is a model
of $P$, it is $M({\cal C}_1 \times \cdots \times {\cal C}_n) \geq F^*$. This
implies that either $M({\cal C}_1) = \cdots = M({\cal C}_{n}) = F^*$ or there
exists $r\leq n$ such that $M({\cal C}_1) = \cdots = M({\cal C}_{r-1}) = F^*$
and $M({\cal C}_r) = T$. By the definition of $N'$, we get in both cases
$N'({\cal C}_1 \times \cdots \times {\cal C}_n) \geq F^*$ (contradiction).

\vspace{0.2cm}
\noindent
{\em Case 3:} $N'(A_1,\ldots,A_m) = T^*$.
Then, by the definition of $N'$, $M(A_1,\ldots,A_m) = T$. Since $M$ is a model
of $P$, it is $M({\cal C}_1 \times \cdots \times {\cal C}_n) = T$. This implies
that there exists some $r$, $1 \leq r \leq n$, such that $M({\cal C}_1) = \cdots
= M({\cal C}_{r-1}) = F^*$ and $M({\cal C}_r) =T$. By the definition of $N'$, we
get that $N'({\cal C}_1 \times \cdots \times {\cal C}_n) \geq T^*$
(contradiction).

\vspace{0.2cm}
\noindent
{\em Case 4:} $N'(A_1,\ldots,A_m) = T$.
Then, by the definition of $N'$, $N(A_1,\ldots,A_m) = T$ and $M(A_1,\ldots,A_m)
= T$. Since $M$ is a model of $P$, it is $M({\cal C}_1 \times \cdots \times
{\cal C}_n) = T$. This implies that there exists some $r$, $1 \leq r \leq n$,
such that $M({\cal C}_1) = \cdots = M({\cal C}_{r-1}) = F^*$ and $M({\cal C}_r)
=T$. By the definition of the reduct, the rule ${\cal C}_r \leftarrow
A_1,\ldots,A_m$ exists in $P^{M}_{\times}$. Since $N$ is a model of
$P^{M}_{\times}$, we get that $N({\cal C}_r)=T$. Thus, $N'({\cal C}_1)= \cdots =
N'({\cal C}_{r-1}) = F^*$ and $N'({\cal C}_r) = T$, and therefore $N'({\cal C}_1
\times \cdots \times {\cal C}_n) = T$ (contradiction).
\end{proof}


%

\ifreview
\clearpage
\input{response}
\fi

\fi

\end{document}